\documentclass[runningheads]{llncs}
\usepackage{graphicx}
\usepackage{subcaption}
\usepackage{amsmath,amssymb} 
\usepackage{wrapfig}
\usepackage{color}
\usepackage{multirow}
\usepackage[width=122mm,left=12mm,paperwidth=146mm,height=193mm,top=12mm,paperheight=217mm]{geometry}
\def\utc{\mbox{$\;\overset{c}{=}\;$}}
\begin{document}
\pagestyle{headings}
\mainmatter
\def\ECCV18SubNumber{2281}  

\newcommand{\FP}[1] {\emph{\textcolor{blue}{(FP: #1)}}}
\newcommand{\IB}[1] {\emph{\textcolor{red}{(IB: #1)}}}
\def\S{\mathbf{S}}

\title{On Regularized Losses \\for Weakly-supervised CNN Segmentation } 

\titlerunning{On Regularized Losses for Weakly-supervised CNN Segmentation}

\authorrunning{M.~Tang, F.~Perazzi, A.~Djelouah, I.~Ben Ayed, C.~Schroers, Y.~Boykov}

\author{Meng Tang$^\dag$ \hspace{3ex}  Federico Perazzi$^*$ \hspace{3ex}  Abdelaziz Djelouah$^*$ \\ 
Ismail Ben Ayed$^\ddag$ \hspace{3ex}    Christopher Schroers$^*$       \hspace{3ex}  Yuri Boykov$^\dag$}
\institute{\footnotesize $^\dag$Computer Science, University of Waterloo, Canada  \\ 
$^*$Disney Research, Z\"{u}rich, Switzerland  \\
$^\ddag$\'Ecole de Technologie Sup\'erieure, University of Quebec, Canada}

\maketitle
\begin{abstract}
Minimization of regularized losses is a principled approach to weak supervision well-established in deep learning, in general.
However, it is largely overlooked in semantic segmentation currently dominated by methods mimicking full
supervision via ``fake'' fully-labeled training masks (proposals) generated from available partial input. To obtain such full masks
the typical methods explicitly use standard regularization techniques for ``shallow'' segmentation, e.g. graph cuts or dense CRFs.
In contrast, we integrate such standard regularizers directly into the loss functions over partial input.
This approach simplifies weakly-supervised training by avoiding extra MRF/CRF inference steps or layers
explicitly generating full masks, while improving both the quality and efficiency of training.
This paper proposes and experimentally compares different losses integrating MRF/CRF regularization terms. We juxtapose our regularized losses with earlier proposal-generation methods using explicit regularization steps or layers. Our approach achieves state-of-the-art accuracy
in semantic segmentation with near full-supervision quality.

\end{abstract}

\section{Introduction}
We advocate {\em regularized loss} functions for weakly-supervised training of semantic CNN segmentation.
The use of unsupervised loss terms acting as regularizers on the output of deep-learning architectures is a principled approach 
to exploit structure similarity of partially labeled data \cite{weston2012deep,goodfellow2016deep}. 
Surprisingly, this general idea was largely overlooked in weakly-supervised CNN segmentation where current methods often
introduce computationally expensive MRF/CRF layers or post-processing inference steps generating ``fake'' full masks 
from partial input. 

We propose to use (relaxations of) MRF/CRF terms directly inside the loss avoiding explicit 
guessing of full training masks.
This approach follows well-established ideas for weak supervision in deep learning \cite{weston2012deep,goodfellow2016deep} and continues
our recent work \cite{ncloss:cvpr18} that proposed the integration of standard objectives in shallow\footnote{In this paper ``shallow''
refers to standard segmentation methods unrelated to CNNs.} segmentation directly into loss functions. While \cite{ncloss:cvpr18}
is entirely focused on the {\em normalized cut loss} motivated by a popular  balanced segmentation criterion \cite{Shi2000},
 we now study a different class of {\em regularized losses} including (relaxations of) standard MRF/CRF potentials.
While they are common as shallow regularizers \cite{BVZ:PAMI01,BJ:01,GrabCuts:SIGGRAPH04,koltun:NIPS11} 
or as trainable layers \cite{zheng2015conditional}, they were never used directly as losses. 

We propose and evaluate several new losses motivated by MRF/CRF potentials
and their combination with balanced partitioning criteria \cite{NC-MRF:eccv16}. Such losses can be adapted to many forms of 
weak (or semi-) supervision based on diverse existing MRF/CRF  formulations for interactive graph cut segmentation.
But, the scope of this paper is limited to training with partial (user scribble) masks where 
regularized losses combined with cross entropy over the partial masks achieve the state-of-the-art close to full-supervision quality.

Besides basic Potts model \cite{BVZ:PAMI01}, we use popular fully connected pairwise CRF potentials of 
Kr{\"{a}}henb{\"{u}}hl and Koltun \cite{koltun:NIPS11}, often referred to as {\em dense CRF}.
In conjunction with CNNs dense CRFs have become the de-facto choice for semantic segmentation
in the contexts of fully \cite{deeplab,CRFlayersjournal,zheng2015conditional} and weakly/semi
\cite{kolesnikov2016seed,papandreou2015weakly,deepcut} supervised learning. 
For instance, DeepLab \cite{deeplab} popularized dense CRF as a post-processing step. 
In fully supervised setting, integrating the unary scores of a CNN classifier and 
the pairwise potentials of dense CRF achieve competitive performances \cite{CRFlayersjournal}. 
This is facilitated by fast mean-field  inference techniques for dense CRF based on high-dimensional filtering \cite{adams2010fast}. 

Weakly supervised semantic segmentation is commonly addressed by mimicking full supervision via 
synthesizing fully-labeled training masks (proposals) from the available partial inputs 
\cite{deepcut,papandreou2015weakly,scribblesup}. 
These schemes typically iterate two steps: CNN training and proposal generation via regularization-based 
shallow interactive segmentation, e.g. graph cut \cite{scribblesup} or dense CRF mean-field inference  \cite{deepcut,papandreou2015weakly}.
In contrast, our approach avoids explicit inference steps by integrating shallow regularizers directly into the loss functions.
Section \ref{sec:proposal-connection} makes some interesting connections between proposal-generation and our regularized losses.

For simplicity, this paper uses a very basic quadratic relaxation of discrete MRF/CRF potentials, even though
there are many alternatives, e.g. TV-based \cite{ChambolleDarbon:TV2009} and convex formulations \cite{Nikolova:SIAM06,pock:CVPR09}, 
$L_p$ relaxations \cite{Couprie:PAMI11}, LP and other relaxations \cite{kumar2016,kumar2017}.
Evaluation of different relaxations in the context of regularized weak supervision losses is left for future work.
Our main contributions are:
\begin{itemize}
\item We propose and evaluate several {\em regularized losses} for weakly supervised CNN segmentation based on
Potts \cite{BVZ:PAMI01}, dense CRF \cite{koltun:NIPS11}, and kernel cut \cite{NC-MRF:eccv16} regularizers (Sec.\ref{sec:overview}).
Our approach avoids explicit inference steps as in proposal-based methods.
This continues the study of losses motivated by standard  shallow segmentation energies started in 
\cite{ncloss:cvpr18} with {\em normalized cut loss}. 
\item We show that iterative proposal-generation schemes for weak supervision, which alternate CNN learning and mean-field inference, 
can be viewed as an approximate alternating direction optimization of regularized losses (Sec.\ref{sec:proposal-connection}).
\item Comprehensive experiments (Sec.\ref{sec:experiments}) with our regularized weakly supervised losses  show
(1) state-of-the-art performance for weakly supervised CNN segmentation reaching near full-supervision accuracy and
(2) better quality and efficiency than proposal generating methods or normalized cut loss \cite{ncloss:cvpr18}.
Alternating schemes (proposal generation) give higher loss at convergence.
\end{itemize}
\section{Our Regularized Semi-supervised Losses}
\label{sec:overview}

This section introduces our regularized losses for weakly-supervised segmentation. 
In general, the use of regularized losses is a well-established approach in semi-supervised deep learning \cite{weston2012deep,goodfellow2016deep}. We advocate this principle for semantic CNN segmentation, 
propose specific shallow regularizers for such losses, and discuss their properties.

Assuming image $I$ and its {\textit{partial}} ground truth labeling or mask $Y$, let $f_\theta(I)$ be the output of a segmentation network parameterized by $\theta$.  In general, CNN training with our joint regularized loss 
corresponds to optimization problem of the following form
\begin{equation}
\min_\theta \ell (f_\theta (I),Y) \;\;+ \;\;\lambda\cdot  R(f_\theta(I))
\end{equation}
where $\ell (S,Y)$ is a ground truth loss and $R(S)$ is a regularization term or regularization loss. 
Both losses have argument $S=f_\theta(I)\in [0,1]^{|\Omega|\times K}$, which is $K$-way softmax segmentation 
generated by a network. Using cross entropy over partial labeling as the ground truth loss, 
we have the following joint \textit{regularized semi-supervised loss}
\begin{equation} \label{eq:jointloss}
\sum_{p\in\Omega_{\cal L}} {H}(Y_p,S_p) \;\;+\;\; \lambda \cdot R(S)
\end{equation}
where $\Omega_{\cal L}\subset \Omega$ is the set of labeled pixels and $H(Y_p,S_p)=-\sum_k-Y_p^k \log S_p^k$ is the cross entropy between network predicted segmentation $S_p\in [0,1]^K$ (a row of matrix $S$ corresponding to point $p$) 
and ground truth labeling $Y_p\in \{0,1\}^K$.

In principle, any differentiable function $R(S)$ can be used as a loss. This paper studies (relaxations of) 
regularizers from shallow segmentation as loss functions. Section \ref{sec:crfloss} details our MRF/CRF loss and its implementation. In Section \ref{sec:kernelcutloss}, we propose {\em kernel cut loss} 
combining CRF with normalized cut terms and justify this combination.

\subsection{Potts/CRF Losses}
\label{sec:crfloss}
Assuming that segmentation variables  $S_p$ are restricted to binary class indicators $S_p\in \{0,1\}^K$,
the standard Potts model \cite{BVZ:PAMI01} could be represented via Iverson brackets  $[\cdot]$,
as on the left hand side below
\begin{equation} \label{eq:relaxation}
\sum_{p,q\in\Omega} W_{pq}\;[S_p \neq S_q]\;\;=\;\;\sum_{p,q\in\Omega} W_{pq}\;\|S_p - S_q\|^2,
\end{equation}
where $W=[W_{pq}]$ is a matrix of pairwise discontinuity costs or an {\em affinity matrix}. 
The right hand side above is a particularly straightforward quadratic relaxation of 
the Potts model that works for relaxed $S_p\in [0,1]^K$ corresponding to a typical soft-max output of CNNs.
In fact, this quadratic function is very common in the general context of regularized weakly supervised losses
in deep learning \cite{weston2012deep}.

As discussed in the introduction, this relaxation is not unique \cite{ChambolleDarbon:TV2009,Nikolova:SIAM06,pock:CVPR09,Couprie:PAMI11,kumar2016}. 
We use slightly different quadratic relaxation of the Potts model
\begin{equation} \label{eq:crfloss}
R_{CRF}(S)=\sum_{k} {S^{k'}{W}({\mathbf1}-S^k)}
\end{equation}
expressed in terms of support vectors for each label $k$, i.e. columns of the segmentation matrix 
$S^k\in [0,1]^{|\Omega|}$. For discrete segment indicators \eqref{eq:crfloss} gives the cost of a cut between segments,
same as the Potts model on the left hand side of \eqref{eq:relaxation}, 
but it differs from the relaxation of the right hand side of \eqref{eq:relaxation}.

The affinity matrix $W$ can be sparse or dense. Sparse $W$ commonly appears in the context of 
boundary regularization and edge alignment in shallow segmentation \cite{BJ:01}. 
With dense Gaussian kernel $W_{pq}$ \eqref{eq:crfloss} is a relaxation of DenseCRF \cite{Kraehenbuehl2013}. 
The implementation details including fast computation of the gradient \eqref{eq:crfgradient}
for CRF loss with dense Gaussian kernel is described in Sec. \ref{sec:experiments}.

\subsection{Kernel Cut Loss}
\label{sec:kernelcutloss}

Besides the CRF loss \eqref{eq:crfloss}, we also propose its combination with normalized cut loss \cite{ncloss:cvpr18} where each term is a ratio of a segment's cut cost (Potts model) over the segment's weighted size (normalization)
\begin{equation} \label{eq:ncloss}
R_{NC}(S)=\sum_{k} \frac{S^{k'}\hat{W}(1-S^k)}{d'S^k},
\end{equation}
where $d=\hat{W}{\bf 1}$ are node degrees. Note that the affinity matrix $\hat{W}$ for normalized cut can be different from $W$ in CRF \eqref{eq:crfloss}.
The combined {\em kernel cut loss} is simply a linear combination of \eqref{eq:crfloss} and \eqref{eq:ncloss}
\begin{equation} \label{eq:kcloss}
R_{KC}(S)=\sum_{k} {S^{k'}{W}({\mathbf1}-S^k)} \;\; + \;\;\gamma \sum_{k} \frac{S^{k'}\hat{W}(1-S^k)}{d'S^k}
\end{equation}
which is motivated by {\em kernel cut} shallow segmentation \cite{NC-MRF:eccv16} with 
complementary benefits of balanced normalized cut  partitioning and  object boundary regularization or edge alignment
as in Potts model. 
While the kernel cut loss is a high-order objective, its gradient \eqref{eq:kcgradient} can be efficiently
implemented, see Sec. \ref{sec:experiments}.



This paper compares experimentally CRF, normalized cut and kernel cut losses for weakly supervised segmentation. In our experiments, the best weakly supervised segmentation is achieved with kernel cut loss.

Note that standard normalized cut and CRF objectives in shallow segmentation 
require fairly different optimization techniques (e.g. spectral relaxation or graph cuts), 
but the standard gradient descent approach for optimizing losses during
CNN training allows significant flexibility in including different regularization terms, as long as there is a reasonable relaxation.

\section{Connecting proposals generation and loss optimization}
\label{sec:proposal-connection}

The main stream of weakly-supervised methods generate segmentation proposals and train with such "fake" ground truth \cite{scribblesup,Chandraker2017,simpledoesit,kolesnikov2016seed,papandreou2015weakly,dai2015boxsup}. In fact, many off-line shallow interactive segmentation techniques can be used to propagate labels and generate masks, e.g. graph cuts \cite{BJ:01,GrabCuts:SIGGRAPH04}, random walker \cite{grady2006random,Couprie:PAMI11}, etc. However, training is vulnerable to mistakes in the proposals. While alternating proposal generation and network training \cite{scribblesup} may improve the quality of the proposals, errors reinforce themselves in such self-taught learning scheme \cite{sslbook}. Rather than training networks to fit potential errors, our regularized semi-supervised loss framework is more direct and principled \cite{sslbook,weston2012deep}.

In this section, we show that proposal methods can be viewed as {\em approximate} alternating direction method (ADM) for optimization \cite{Boyd2011}, which does not account directly for network variables $\theta$ in the ADM splitting. This optimization insight suggests that expressing very popular regularization terms, for instance, dense CRF, explicitly in terms of the network variables and performing direct back-propagation could be a better optimization alternative to the existing proposal generation methods, in both the quality of the obtained solutions and efficiency. Our optimization results confirm this, e.g. see the CRF loss plot in Fig. \ref{fig:directlossvstrustregion} and the training times in Table \ref{tab:sectag}.

We consider proposal-generation schemes iterating between two steps, \textbf{network training} and \textbf{proposal generation}. 
Then alternation can happen either when training converges or online for each batch. At each iteration,
the first step learns the network parameters $\theta$  from a given (fixed) ground-truth proposal $\tilde{X}$ computed at the previous iteration. 
This amounts to updating the K-way softmax segmentation $S$ to $\tilde{S} \equiv f_{\tilde{\theta}}(I)$ by minimizing the following 
proposal-based cross entropy with respect to parameters $\theta$ via standard back-propagation:
\begin{equation} \label{eq:proposal-cross-entropy}
\tilde{\theta} = \arg \min_{\theta}   \sum_{p\in\Omega_{{\cal L}}} {H}(Y_p,S_p) + \sum_{p\in\Omega_{{\cal U}}} {H}(\tilde{X}_p,S_p) \;\;\;\;\;\text{for}\;\;\;\;\;
S \equiv f_{\theta}(I)
\end{equation}
where $\tilde{X}_p \in [0,1]^K$ are the ground-truth proposals for unlabeled pixels $p \in \Omega_{{\cal U}}$. 
Mask $\tilde{X}_p$ is constrained to be equal to $Y_p$ for labeled pixels $p \in \Omega_{\cal L}$.
The second step fixes the network output $\tilde{S}$ and finds the next ground-truth proposal by minimizing regularization functionals
that are standard in shallow segmentation:
\begin{equation}
\label{proposal-CRF}
\min_{X \in [0,1]^{|\Omega|\times K}} \sum_{p\in\Omega_{{\cal U}}} {H}(X_p,\tilde{S}_p) + \lambda R(X)
\end{equation}
where $X_p \in [0,1]^k$ denotes latent pixel labels within the probability simplex.
Note that for fixed $\tilde{S}$ the cross entropy terms ${H}(X_p,\tilde{S}_p)$ in \eqref{proposal-CRF} are unary potentials for $X$.
When $R$ corresponds to dense CRF, optimization of \eqref{proposal-CRF} is facilitated by fast mean-field  inference
techniques \cite{koltun:NIPS11,Baque2016} significantly reducing the computational times via parallel updates of  variables $X_p$ 
and high-dimensional filtering \cite{adams2010fast}. 
Appendix \ref{sec:meanfield} shows that mean-field algorithms can be equivalently interpreted as a {\em convex-concave} approach
to optimizing the following objective 
\begin{equation}
\label{proposal-CRF-negative-entropie}
\min_{X \in [0,1]^{|\Omega|\times K}}  \sum_{p\in\Omega_{{\cal U}}} {H}(X_p,\tilde{S}_p) + \lambda R(X) -\sum_{p\in\Omega_{{\cal U}}} H(X_p)
\end{equation}
combining \eqref{proposal-CRF} and negative entropies $H(X_p) = -\sum_k X_p^k \log X_p^k$ 
that act as a simplex {\em barrier} for variables $X_p$. This yields closed-form independent (parallel) 
updates of variables $X_p$, while ensuring convergence under some 
conditions\footnote{Parallel updates are guaranteed to converge for concave CRF models, e.g. Potts \cite{Kraehenbuehl2013}.}.

\begin{proposition}
Proposal methods alternating steps \eqref{proposal-CRF-negative-entropie} and \eqref{eq:proposal-cross-entropy} can be viewed 
as approximate {\em alternating direction method} (ADM)\footnote{In its basic form, {\em alternating direction method} 
transforms problem $\min_x f(x) + g(x)$ into $\min_{x,y} f(x) + g(y) \; \mbox{s.t} \; x=y$ and alternates optimization over $x$ and $y$. 
This may work if optimizing $f$ and $g$ seperately is easier than the original problem.} 
\cite{Boyd2011} for optimizing our {\em regularized loss}  \eqref{eq:jointloss} using the following decomposition of the problem:
\begin{equation}
\label{Alternating-direction-KL}
\min_{\theta, X \in [0,1]^{|\Omega|\times K}}  \sum_{p\in\Omega_{{\cal L}}} {H}(Y_p,S_p) + \lambda R(X) + \sum_{p\in\Omega_{{\cal U}}} KL(X_p|S_p)
\end{equation}
where $KL$ denotes the Kullback-Leibler divergence.
\end{proposition}
\begin{proof}
The link between \eqref{Alternating-direction-KL} and \eqref{proposal-CRF-negative-entropie} comes directly from the following relation between the KL divergence and the entropies: $KL(X_p|S_p) = {H}(X_p, S_p) - H(X_p)$.
\end{proof}
Instead of optimizing directly regularized loss \eqref{eq:jointloss} with respect to network parameters, proposal methods splits 
the optimization problem into two easier sub-problems in \eqref{Alternating-direction-KL}. 
This is done by replacing the network softmax outputs $S_p$ in the regularization by latent distributions $X_p$ (the proposals) and 
minimizing a divergence between $S_p$ and $X_p$, which is KL in this case. This is conceptually similar to the general principles of 
ADM \cite{Boyd2011}, except that the splitting is not done directly with respect the variables of the problem (i.e., parameters $\theta$) 
but rather with respect to network outputs $S$.
This can be viewed as an {\em approximate} ADM scheme, which does not account directly for variables $\theta$ in the ADM splitting. 
Note that the method in \cite{kolesnikov2016seed} generates proposals via dense CRF layer, 
but their approach slightly deviates from the described ADM scheme since they also back-propagate 
through this layer\footnote{Cross-entropy loss $H(X(S),S)$ in \cite{kolesnikov2016seed} uses CRF layer proposal $X(S)$ generated 
from network output $S$. Dependence of $X$ on $S$ motivates back-propagation for this layer.}. 
But, as we show in Table \ref{tab:sectag}, 
such back-propagation does not help and can be dropped. Moreover, our direct optimization of regularized losses makes such proposal generating
layers (or procedures) entirely redundant. Our approach gives simpler and more efficient training avoiding expensive iterative inference
\cite{kolesnikov2016seed} and obtaining better performance.


\section{Experiments} \label{sec:experiments}
Sec. \ref{sec:comparelosses} is the main experimental result of this paper. For weakly-supervised segmentation with scribbles \cite{scribblesup}, we train with different regularized losses. The experiments cover our proposed CRF loss, high-order normalized cut loss in \cite{ncloss:cvpr18} and kernel cut loss, as discussed in Sec. \ref{sec:overview}. We show that combining CRF \eqref{eq:crfloss} with normalized cut \eqref{eq:ncloss} \textit{a la} KernelCut \cite{NC-MRF:eccv16} yields the best performance.

In Sec. \ref{sec:directvsproposal}, using direct loss and using generated proposals for training are compared. In the light of the technical connection of the two schemes from optimization perspective in Sec. \ref{sec:proposal-connection}, we also evaluate how ``regularized'' are the segmentations obtained by computing the regularization energy. Besides for scribbles, we also utilize our regularized loss framework for image-level labels based supervision and compare to SEC \cite{kolesnikov2016seed}, a recent method based on proposal generation. Our method achieved the state-of-the-art for weakly supervised segmentation with scribbles or image-level labels.

We also investigate if regularization loss will facilitate fully or semi-supervised segmentation with unlabeled images. Some preliminary results are given in Sec. \ref{sec:fullyandsemi} for these extensions.

\hspace{-.5cm}\textbf{Dataset} Most experiments are on the PASCAL VOC12 segmentation dataset. For all method, we train with the augmented dataset of 10,582 images. The scribble annotations for these training images are from \cite{scribblesup}. Following standard protocol, mean intersection-over-union (mIoU) over the 21 classes is evaluated on the \textit{val} set that contains 1,449 images. For image-level label supervision, our experiment setup and dataset is the same as that used in \cite{kolesnikov2016seed}.

\hspace{-.5cm}\textbf{Implementation details} Our implementation is based on DeepLab v2 \cite{deeplab}. We follow the learning rate strategy in DeepLab v2 \footnote{https://bitbucket.org/aquariusjay/deeplab-public-ver2} for the baseline with full supervision. For our method with regularized loss, we first train with partial cross entropy loss only for the seeds. Then we fine-tune with extra regularized losses of different types for the same number of iterations. 
Our CRF and normalized cut regularization losses are defined at full image resolution. If the network outputs shrinked labeling, which is typical, the labeling is interpolated to original resolution before feeding into the loss layer.

We choose dense Gaussian kernel over RGBXY channels for $R_{CRF}(S)$, $R_{NC}(S)$ and $R_{KC}(S)$. As hyper-parameter, the Gaussian bandwidth is optimized via validation for DenseCRF, normalized cut and kernel cut. As is also mentioned in \cite{ncloss:cvpr18}, naive forward and backward pass of such fully-connected pairwise or high-order loss layer would be prohibitively slow ($O(|\Omega|^2)$ for $|\Omega|$ pixels). For example, to implement $R_{CRF}(S)$ \eqref{eq:crfloss} as a loss, we need to compute its gradient w.r.t. $S^k$ during backpropagation,

\begin{equation} \label{eq:crfgradient}
\frac{\partial R_{CRF}(S)}{\partial S^k}=-2WS^k.
\end{equation}

For DenseCRF where $W$ is fully connected Gaussian, computing the gradient \eqref{eq:crfgradient} becomes a standard Bilateral filtering problem, for which many fast methods were proposed \cite{adams2010fast,paris2009fast}. We implement our loss layers using fast Gaussian filtering \cite{adams2010fast}, which is also utilized in the inference of DenseCRF \cite{koltun:NIPS11,zheng2015conditional}. Using the same fast filtering component, we can also computer the following gradient \eqref{eq:kcgradient} of our Kernel Cut loss \eqref{eq:kcloss} in linear time. Note that our CRF and KC loss layer is much faster than CRF inference layer \cite{kolesnikov2016seed,zheng2015conditional} since no iterations is needed.

\begin{equation} \label{eq:kcgradient}
\frac{\partial R_{KC}(S)}{\partial S^k}=-2WS^k+\gamma \frac{S^{k'}\hat{W}S^kd}{(d'S^k)^2}-\gamma\frac{2\hat{W}S^k}{d'S^k}.
\end{equation}

\subsection{Comparison of regularized losses}
\label{sec:comparelosses}

\begin{table}[t]
\centering
\begin{tabular}{l | c|c|c|c|c}
\hline
   &   \multicolumn{4}{c|}{Weak}  &\multirow{2}{*}{ Full}\\
           & CE only & w/ NC \cite{ncloss:cvpr18} & w/ CRF & w/ KernelCut &\\\hline
{DeepLab-MSc-largeFOV}                   & 56.0 (8.1)& 60.5 (3.6) & 63.1 (1.0)& \textbf{63.5 (0.6)}& 64.1\\
{DeepLab-MSc-largeFOV+CRF}             &  62.0 (6.7) &  65.1 (3.6) & 65.9 (2.8) & \textbf{66.7 (2.0)}&      68.7 \\
{DeepLab-VGG16}                   & 60.4 (8.4) & 62.4 (6.4) & 64.4 (4.4) & \textbf{64.8 (4.0)}& 68.8 \\
{DeepLab-VGG16+CRF}               &  64.3 (7.2) & 65.2 (6.3) & 66.4 (5.1) & \textbf{66.7 (4.8)}&     71.5\\
{DeepLab-ResNet101  }             & 69.5 (6.1)& 72.8 (2.8) & 72.9 (2.7)&\textbf{73.0 (2.6)}& 75.6\\
{DeepLab-ResNet101+CRF}              & 72.8 (4.0)& 74.5 (2.3) & \textbf{75.0 (1.8)}& \textbf{75.0 (1.8)} & 76.8 \\\hline
\end{tabular}
\caption{mIOU on PASCAL VOC2012 {\textit val} set. Our flexible framework allows various types of regularization losses for weakly supervised segmentation, e.g. normalized cut, CRF or their combinations (KernelCut \cite{NC-MRF:eccv16}) as joint loss. We achieved  the state-of-the-art with scribbles. In () shows the offset to the result with full masks.}
\label{tab:mainresult}
\end{table}

Tab. \ref{tab:mainresult} summaries the results with different regularized losses. Here we report both result with or without CRF post-processing on various networks. The baselines are with cross entropy losses of full labeled masks or partial seeds ignoring unlabeled region. We choose the weight of the regularization term to achieve the best validation accuracy. The state-of-the-art of scribble-based segmentation is from prior work \cite{ncloss:cvpr18} with extra normalized cut loss. Consistently over different networks, using the proposed CRF loss outperforms that with normalized cut loss. Our best result is obtained when combining both normalized cut loss and DenseCRF loss. Clearly, utilization of CRF loss and KernelCut loss reduce the gap toward the full supervision baseline. With DeepLab-MSc-largeFOV followed by CRF post processing, using KernelCut regularized loss achieved mIOU of 66.7\%, while previous best is 65.1\% with normalized cut loss \cite{ncloss:cvpr18}. Our result with scribbles approaches \textbf{97.6\%} of the quality of that with full supervision, yet only ~\textbf{3\%} of all pixels are scribbled. This paper pushes the limit of weakly supervised segmentation.

\begin{figure}[h]
    \centering
    \begin{subfigure}{0.125\textwidth}
    \centering
        \includegraphics[width=1\columnwidth]{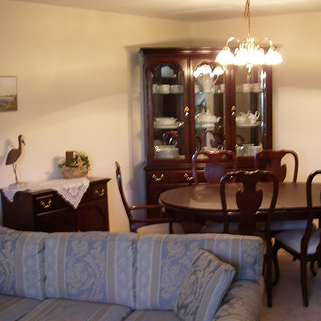}\\
        \includegraphics[width=1\columnwidth]{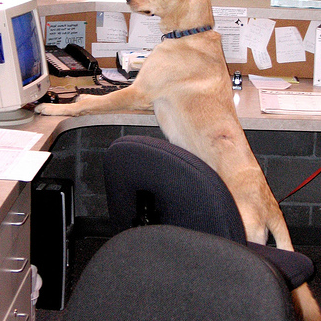} \\
        \includegraphics[width=1\columnwidth]{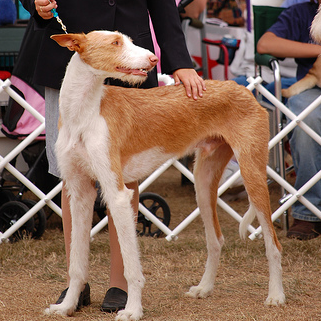} \\
        \includegraphics[width=1\columnwidth]{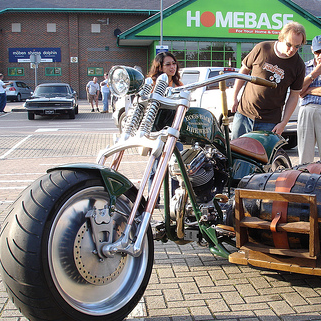}
        \captionsetup{labelformat=empty}
        \caption{image}
    \end{subfigure}
            \begin{subfigure}{0.17\textwidth}
    \centering
        \includegraphics[width=1\columnwidth,trim={2cm 1cm 0.5cm 1cm},clip]{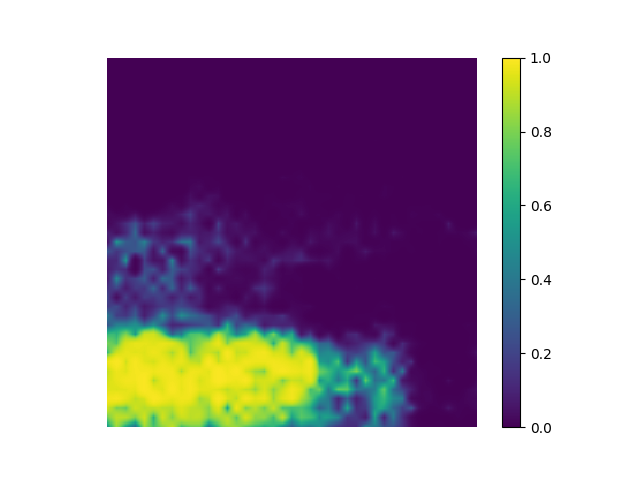}\\
        \includegraphics[width=1\columnwidth,trim={2cm 1cm 0.5cm 1cm},clip]{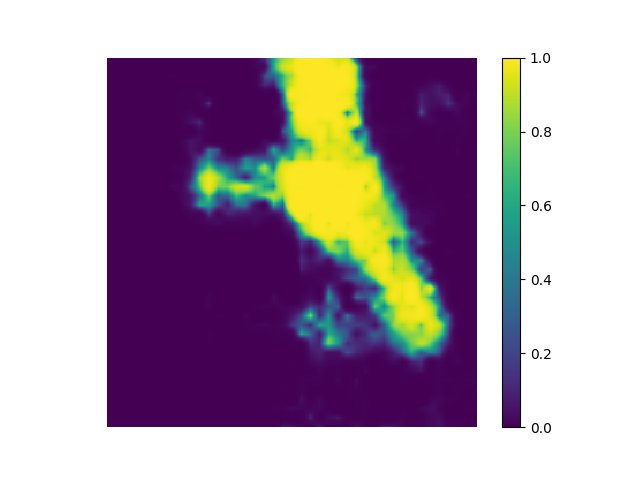} \\
        \includegraphics[width=1\columnwidth,trim={2cm 1cm 0.5cm 1cm},clip]{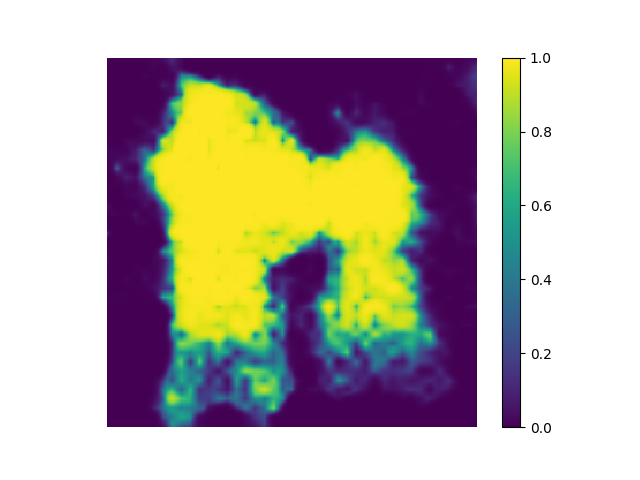} \\
        \includegraphics[width=1\columnwidth,trim={2cm 1cm 0.5cm 1cm},clip]{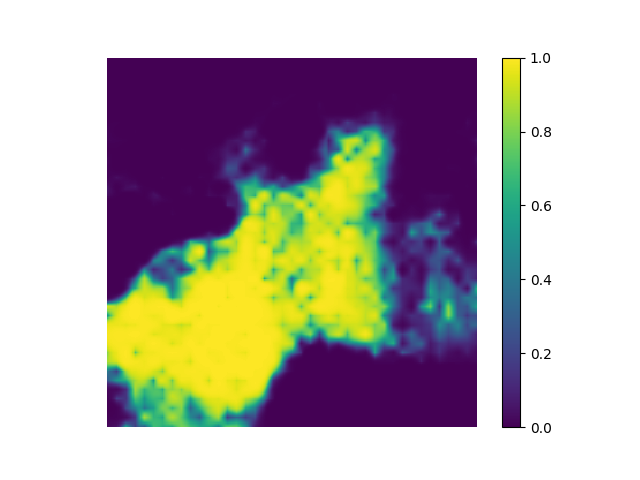}
        \captionsetup{labelformat=empty}
        \caption{{\scriptsize network output}}
    \end{subfigure}
        \begin{subfigure}{0.17\textwidth}
    \centering
        \includegraphics[width=1\columnwidth,trim={2cm 1cm 0.5cm 1cm},clip]{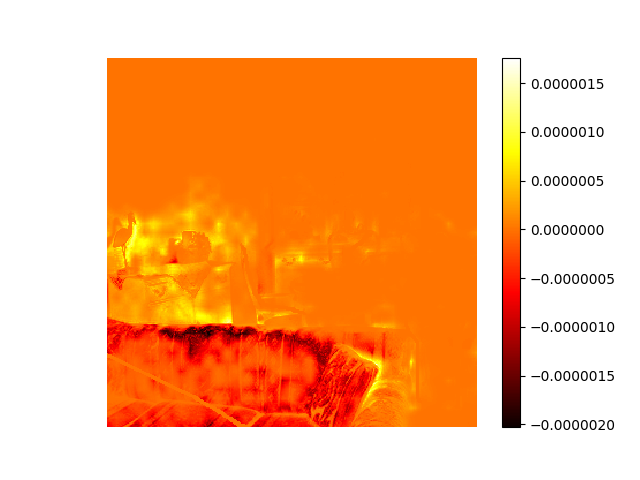}\\
        \includegraphics[width=1\columnwidth,trim={2cm 1cm 0.5cm 1cm},clip]{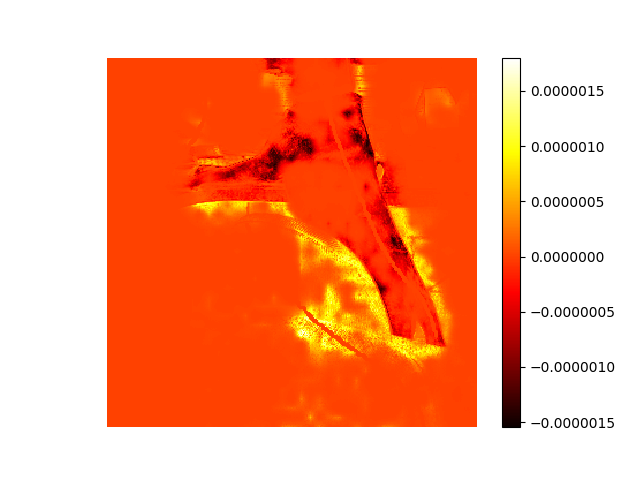} \\
        \includegraphics[width=1\columnwidth,trim={2cm 1cm 0.5cm 1cm},clip]{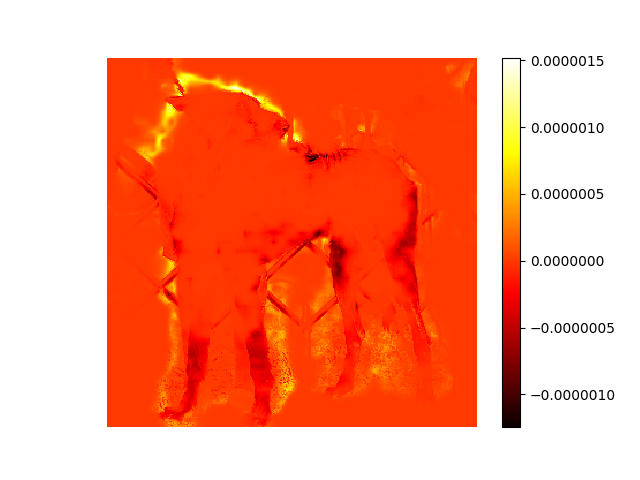} \\
        \includegraphics[width=1\columnwidth,trim={2cm 1cm 0.5cm 1cm},clip]{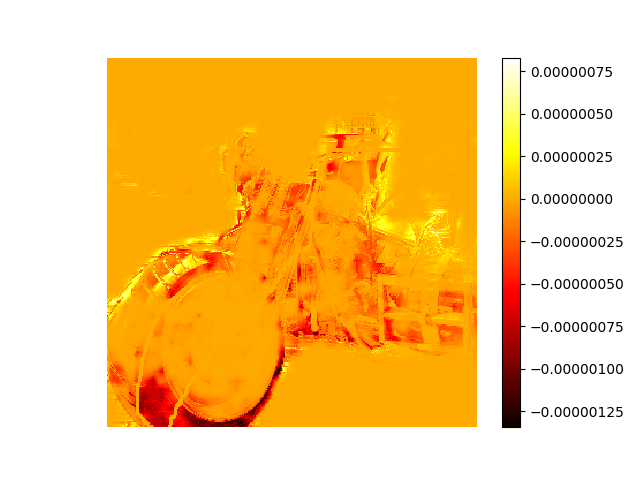}
        \captionsetup{labelformat=empty}
        \caption{NC grad.}
    \end{subfigure}
        \begin{subfigure}{0.17\textwidth}
    \centering
        \includegraphics[width=1\columnwidth,trim={2cm 1cm 0.5cm 1cm},clip]{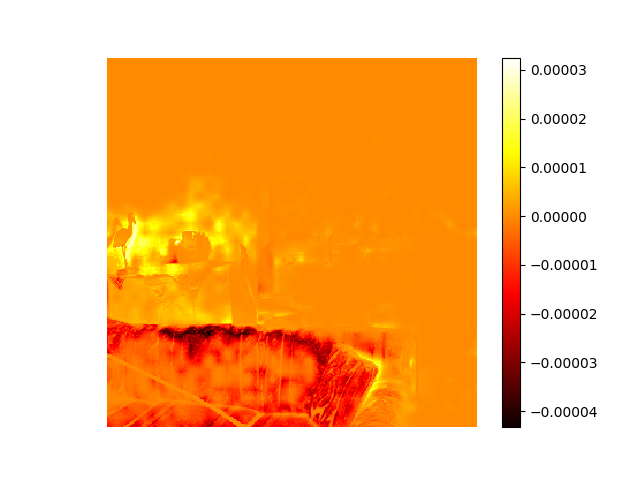}\\
        \includegraphics[width=1\columnwidth,trim={2cm 1cm 0.5cm 1cm},clip]{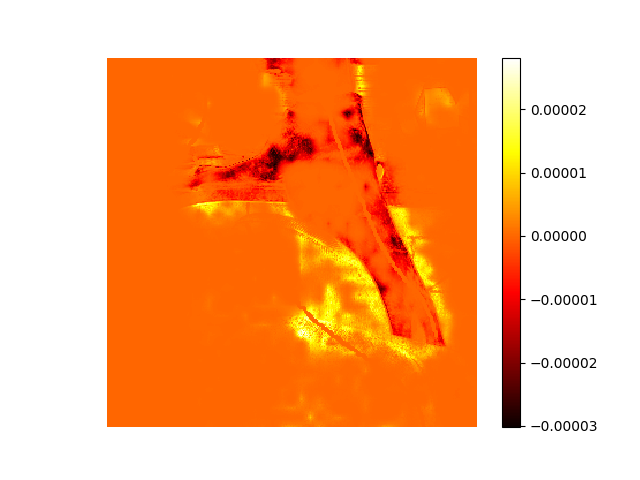} \\
        \includegraphics[width=1\columnwidth,trim={2cm 1cm 0.5cm 1cm},clip]{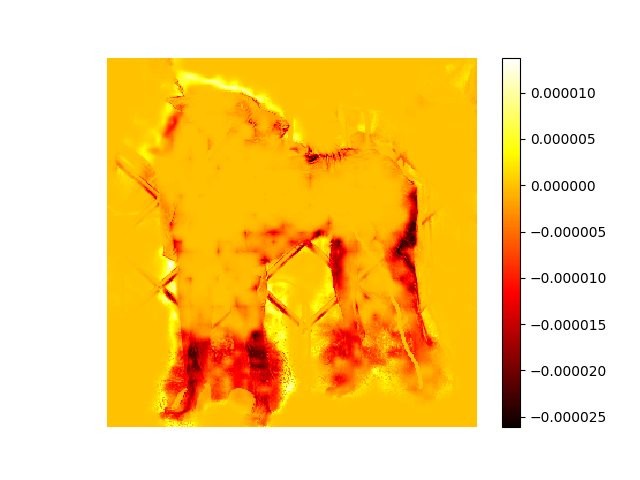} \\
        \includegraphics[width=1\columnwidth,trim={2cm 1cm 0.5cm 1cm},clip]{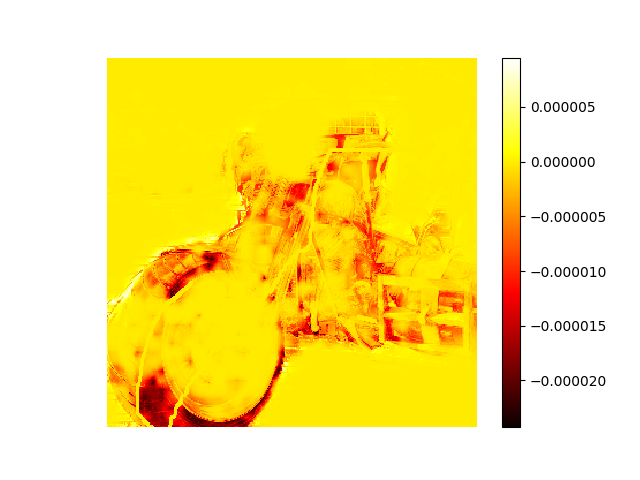}
        \captionsetup{labelformat=empty}
        \caption{CRF grad.}
    \end{subfigure}
     \begin{subfigure}{0.17\textwidth}
    \centering
        \includegraphics[width=1\columnwidth,trim={2cm 1cm 0.5cm 1cm},clip]{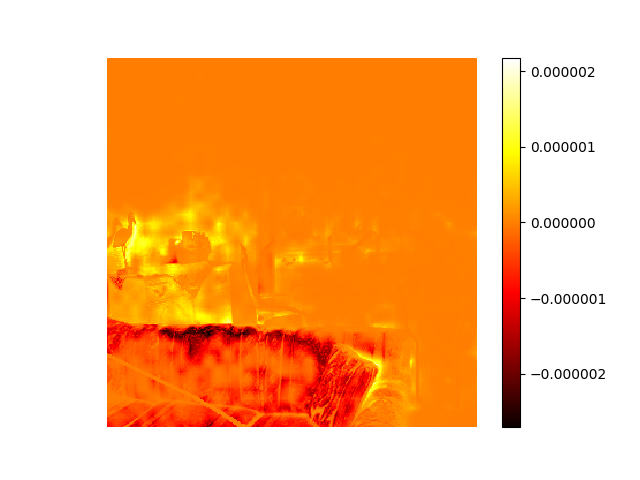}\\
        \includegraphics[width=1\columnwidth,trim={2cm 1cm 0.5cm 1cm},clip]{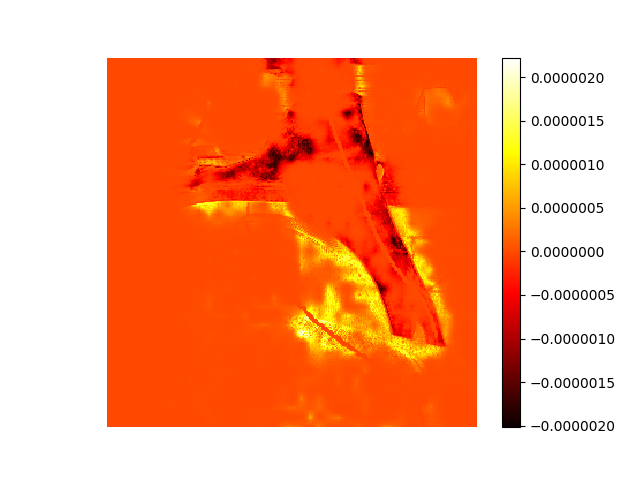} \\
        \includegraphics[width=1\columnwidth,trim={2cm 1cm 0.5cm 1cm},clip]{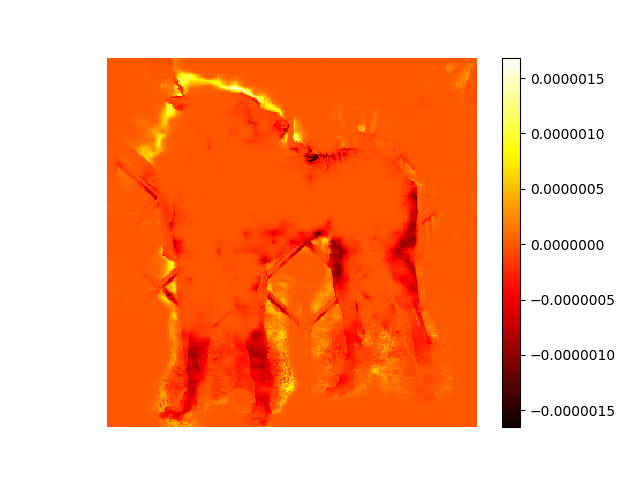} \\
        \includegraphics[width=1\columnwidth,trim={2cm 1cm 0.5cm 1cm},clip]{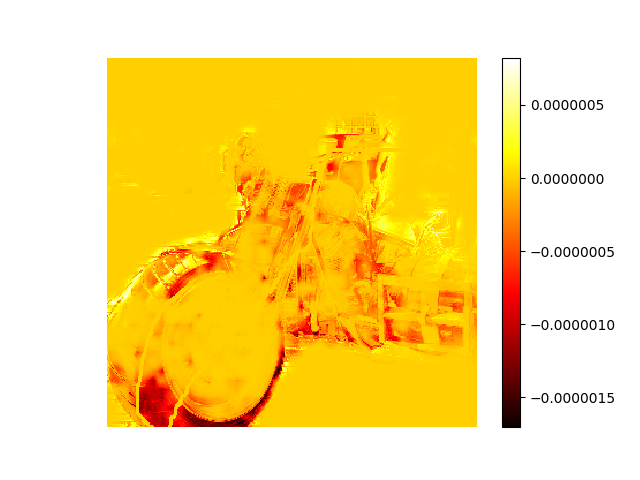}
        \captionsetup{labelformat=empty}
        \caption{KC grad.}
    \end{subfigure}

    \caption{Visualization of the gradient for different losses. The negative (positive) gradients are coded in red (yellow). For example, negative gradients on the sofa drives the network to predict ``sofa'' for these pixels. Also note how the dog pops out in the gradient map.}
   \label{fig:gradient}
\end{figure}

To get some intuition about these losses and their regularization effect, we visualize their gradient w.r.t. segmentation $\frac{\partial R(S)}{\partial S}$ in Fig. \ref{fig:gradient}. Note that the \textit{sign} of gradients indicates whether to encourage or discourage certain labeling. The color coded gradients clearly show evidence toward better color clustering /edge alignment/ object separation with regularized loss. The gradients of different losses are slightly different. Since kernel cut is the combination of normalized cut with CRF, then its gradient is the sum of that of each.

Fig. \ref{fig:scribblesexamples} shows some qualitative examples with different losses. Results with regularized loss is better than that without. Besides, the segmentation with kernel cut loss have better edge alignment compared to that with normalized cut loss. This is because of the extra pairwise CRF loss. The effect of CRF loss and normalized cut loss is different. Our Kernel Cut loss combines the benefit of both regional color clustering (normalized cut) and pairwise regularization (DenseCRF). By combining both we can achieve better segmentation regularization.

\begin{figure}[h!]
    \centering
    \begin{subfigure}{0.16\textwidth}
    \centering
        \includegraphics[width=1\columnwidth]{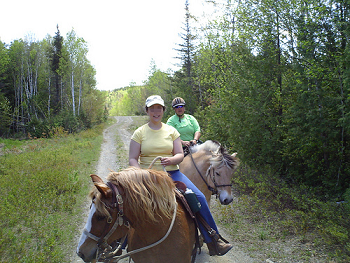}\\
        \includegraphics[width=1\columnwidth]{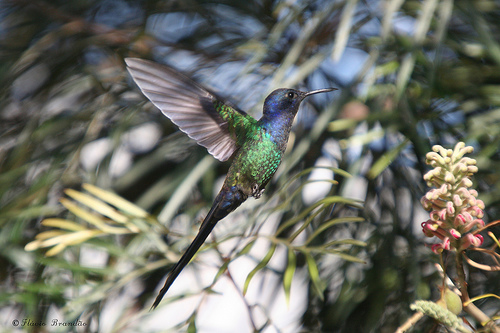} \\
         \includegraphics[width=1\columnwidth]{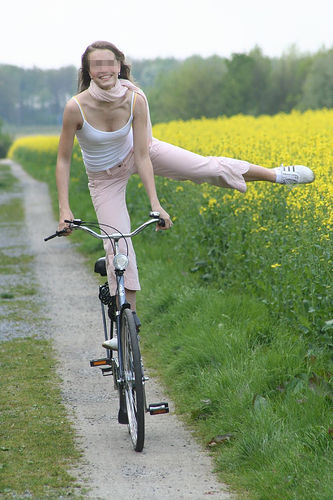} \\
        \includegraphics[width=1\columnwidth]{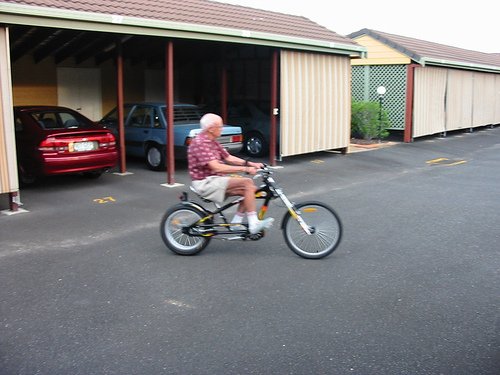} \\
        \includegraphics[width=1\columnwidth]{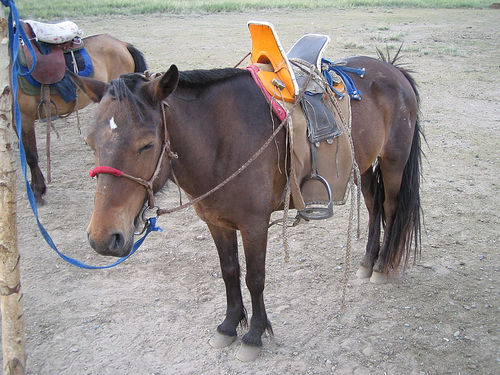} \\
        \includegraphics[width=1\columnwidth]{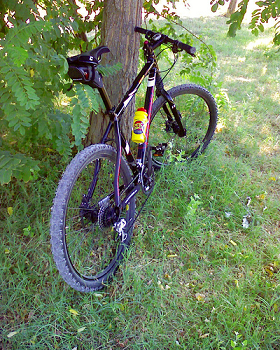} \\
        \includegraphics[width=1\columnwidth]{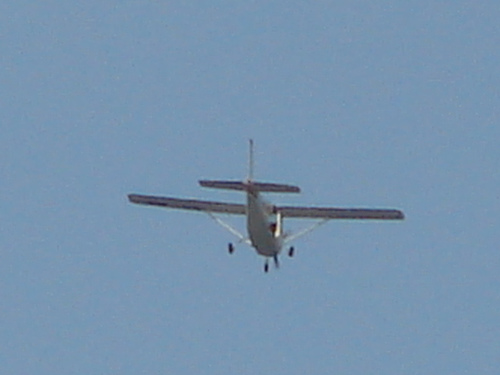} \\
        \includegraphics[width=1\columnwidth]{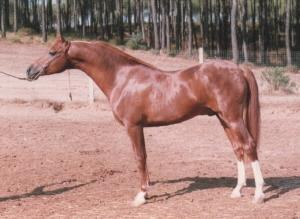}\\
        \includegraphics[width=1\columnwidth]{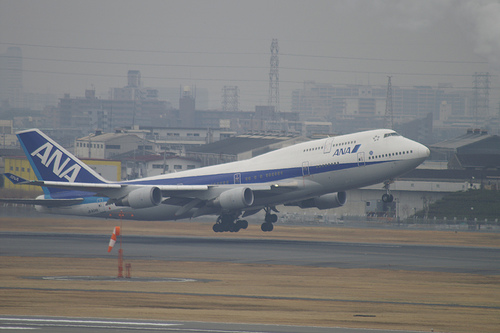}
        \captionsetup{labelformat=empty}
        \caption{image}
    \end{subfigure}
            \begin{subfigure}{0.16\textwidth}
    \centering
        \includegraphics[width=1\columnwidth]{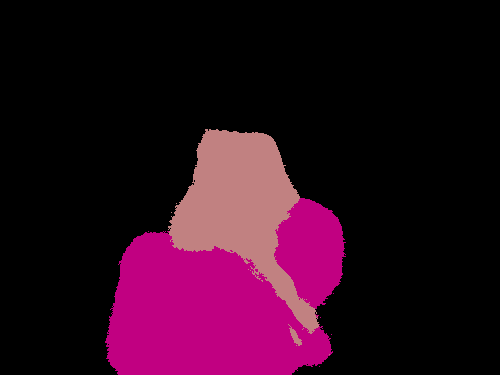}\\
        \includegraphics[width=1\columnwidth]{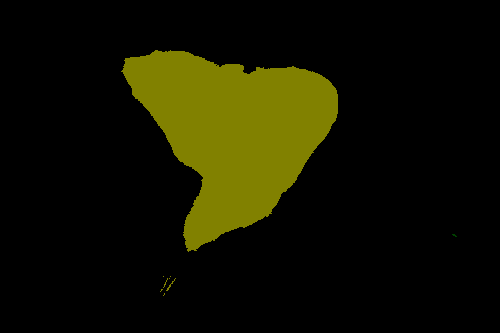} \\
         \includegraphics[width=1\columnwidth]{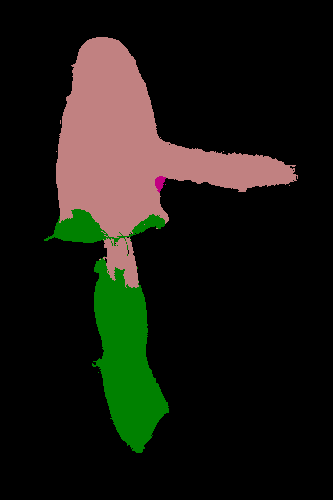} \\
        \includegraphics[width=1\columnwidth]{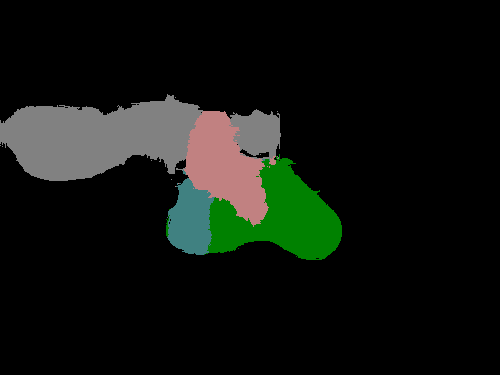} \\
        \includegraphics[width=1\columnwidth]{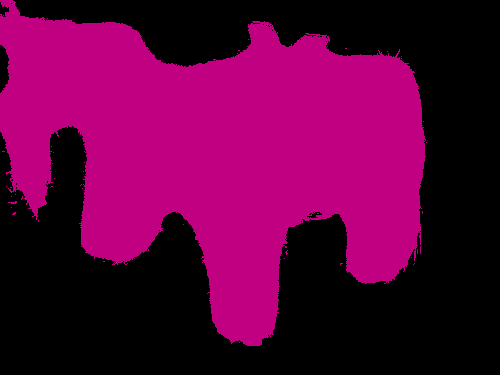} \\
        \includegraphics[width=1\columnwidth]{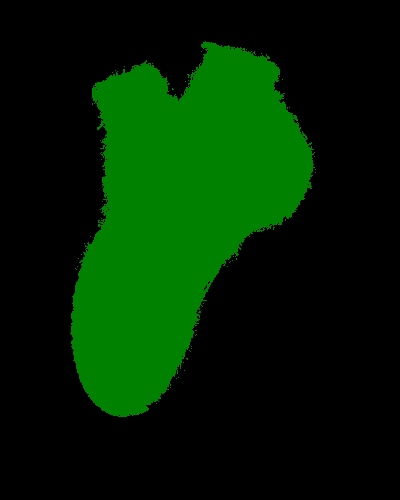} \\
        \includegraphics[width=1\columnwidth]{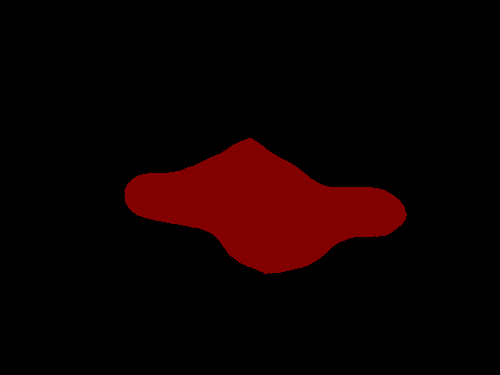} \\
        \includegraphics[width=1\columnwidth]{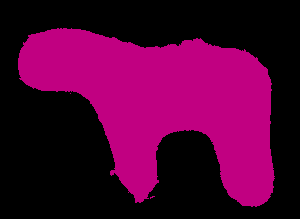}\\
        \includegraphics[width=1\columnwidth]{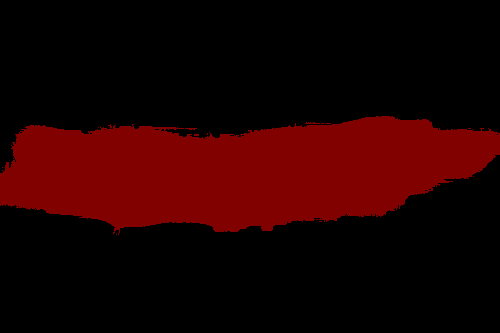}
        \captionsetup{labelformat=empty}
        \caption{CE loss only}
    \end{subfigure}
        \begin{subfigure}{0.16\textwidth}
    \centering
        \includegraphics[width=1\columnwidth]{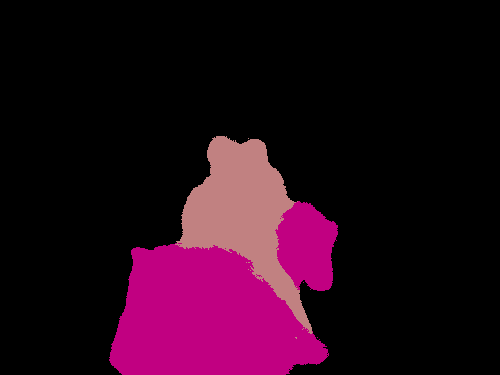}\\
        \includegraphics[width=1\columnwidth]{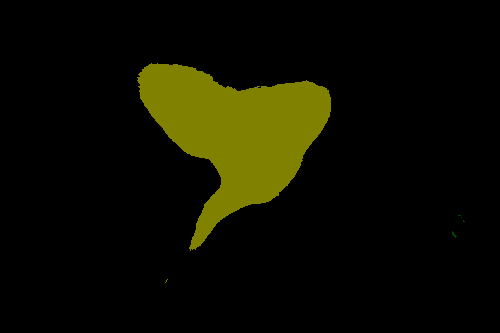} \\
         \includegraphics[width=1\columnwidth]{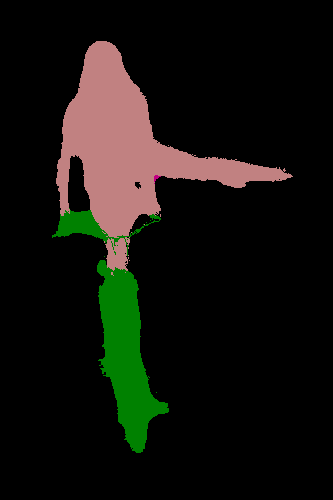} \\
        \includegraphics[width=1\columnwidth]{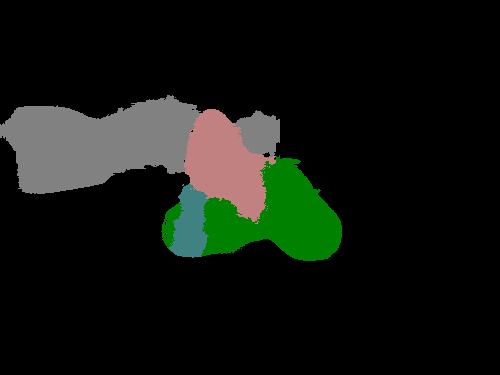} \\
        \includegraphics[width=1\columnwidth]{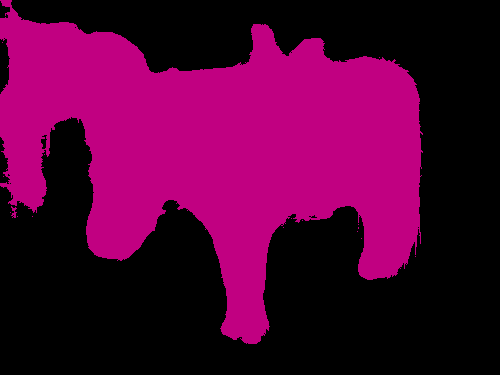} \\
        \includegraphics[width=1\columnwidth]{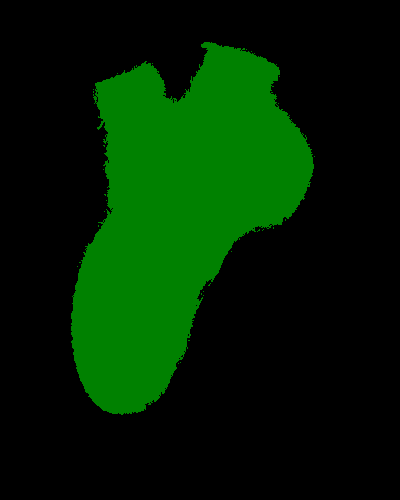} \\
        \includegraphics[width=1\columnwidth]{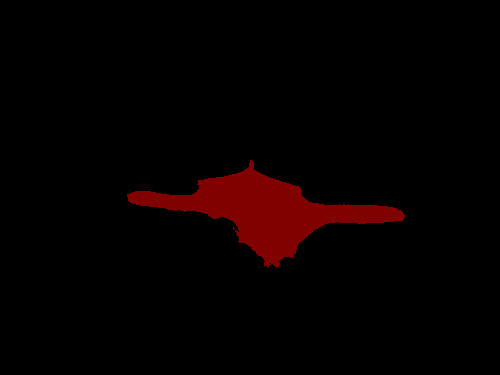} \\
        \includegraphics[width=1\columnwidth]{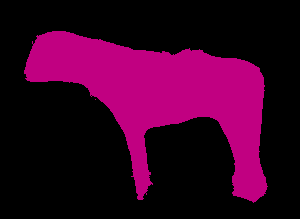}\\
        \includegraphics[width=1\columnwidth]{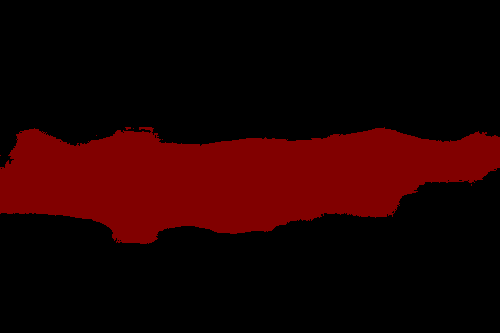}
        \captionsetup{labelformat=empty}
        \caption{w/ NC loss \cite{ncloss:cvpr18}}
    \end{subfigure}
     \begin{subfigure}{0.16\textwidth}
    \centering
        \includegraphics[width=1\columnwidth]{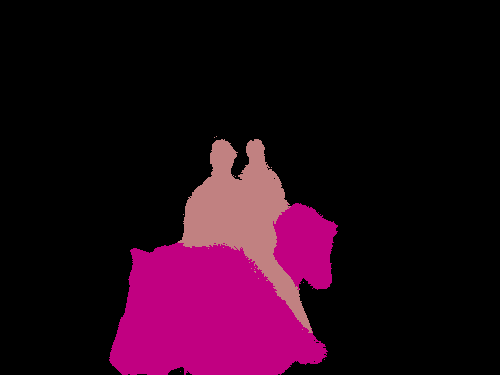}\\
        \includegraphics[width=1\columnwidth]{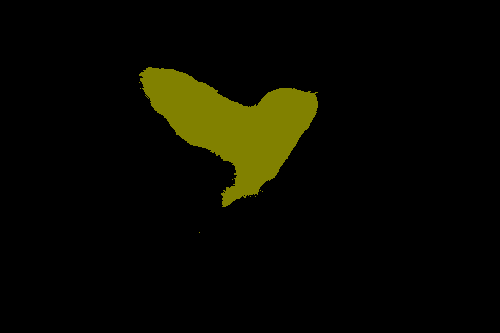} \\
         \includegraphics[width=1\columnwidth]{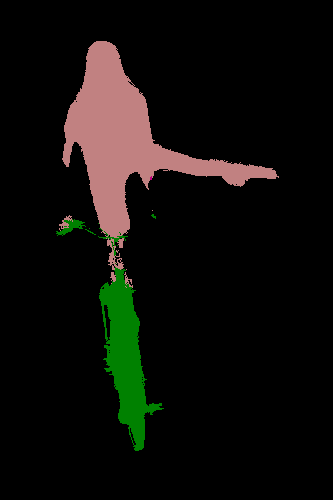} \\
        \includegraphics[width=1\columnwidth]{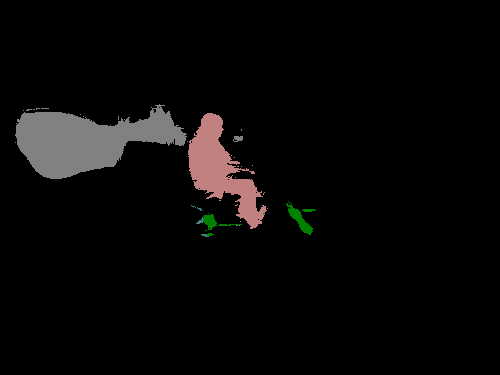} \\
        \includegraphics[width=1\columnwidth]{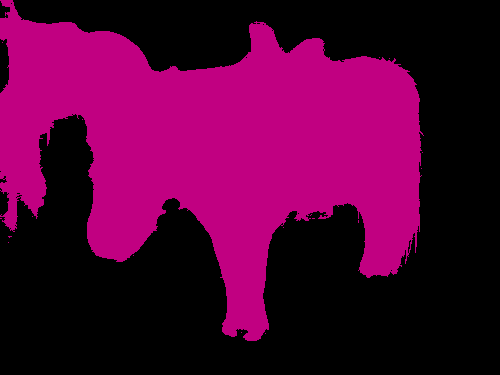} \\
        \includegraphics[width=1\columnwidth]{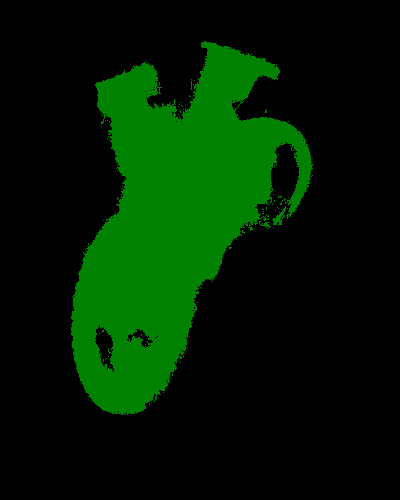} \\
        \includegraphics[width=1\columnwidth]{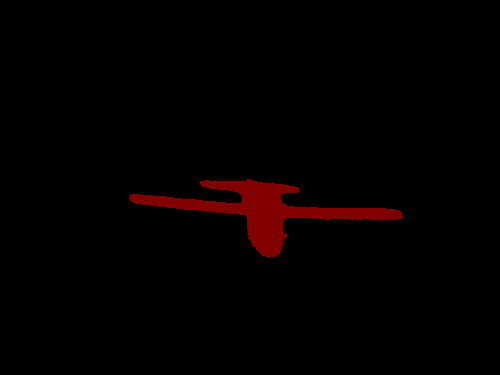} \\
        \includegraphics[width=1\columnwidth]{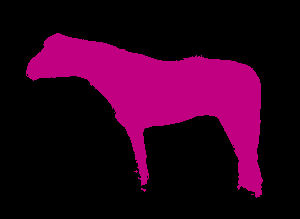}\\
        \includegraphics[width=1\columnwidth]{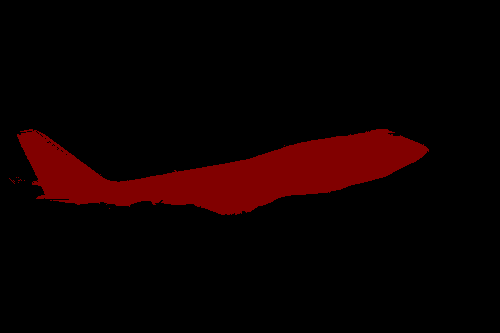}
        \captionsetup{labelformat=empty}
        \caption{w/ CRF loss}
    \end{subfigure}
        \begin{subfigure}{0.16\textwidth}
    \centering
        \includegraphics[width=1\columnwidth]{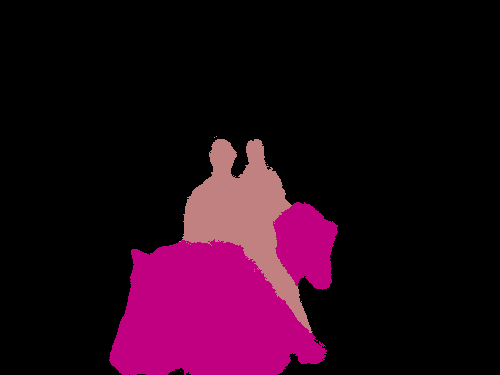}\\
        \includegraphics[width=1\columnwidth]{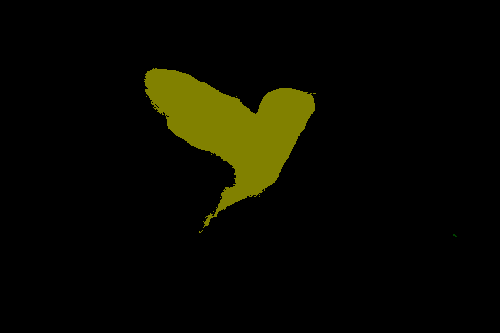} \\
         \includegraphics[width=1\columnwidth]{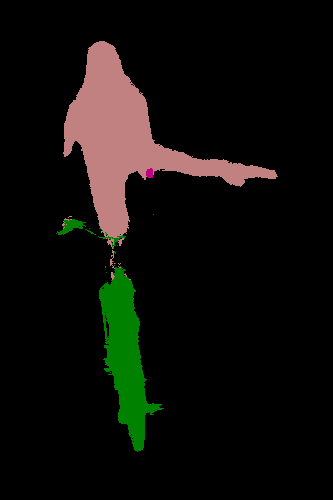} \\
        \includegraphics[width=1\columnwidth]{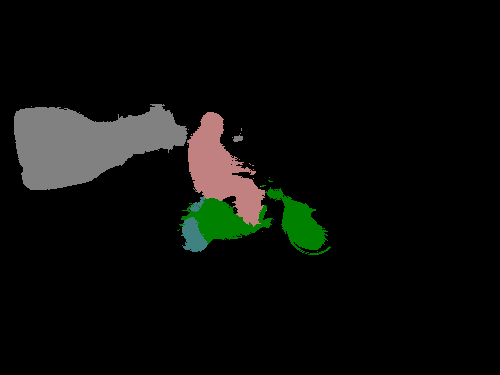} \\
        \includegraphics[width=1\columnwidth]{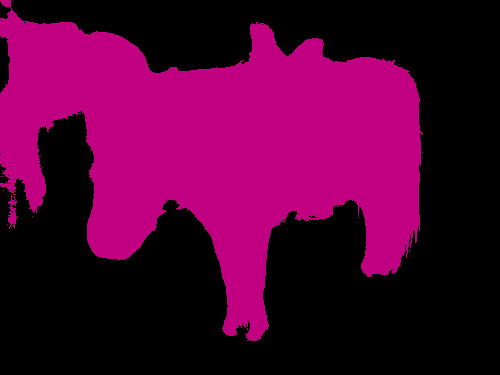} \\
        \includegraphics[width=1\columnwidth]{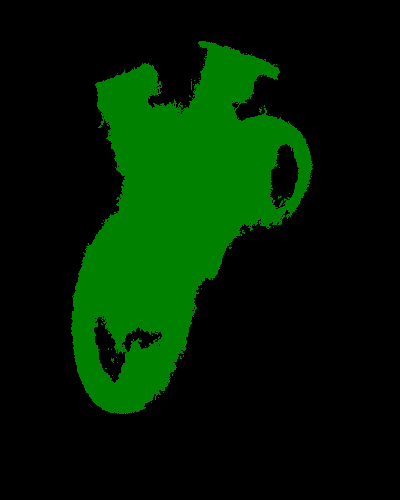} \\
        \includegraphics[width=1\columnwidth]{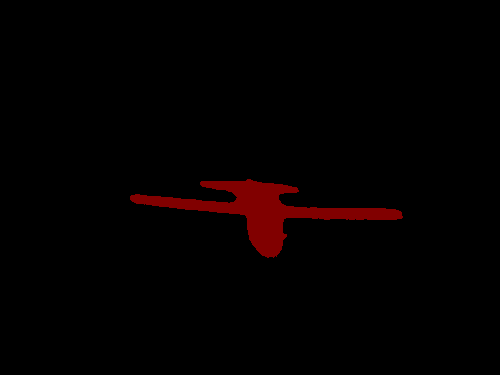} \\
        \includegraphics[width=1\columnwidth]{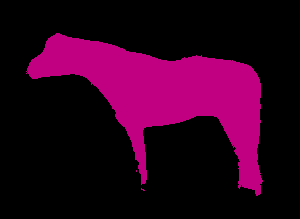}\\
        \includegraphics[width=1\columnwidth]{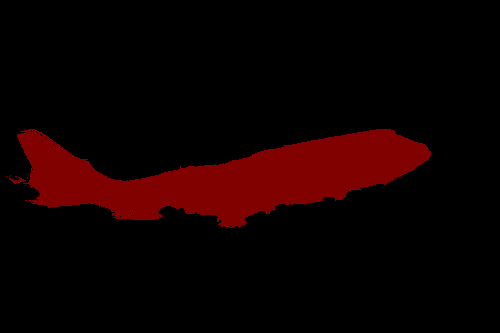}
        \captionsetup{labelformat=empty}
        \caption{{\scriptsize KernelCut loss}}
    \end{subfigure}
        \begin{subfigure}{0.16\textwidth}
    \centering
        \includegraphics[width=1\columnwidth]{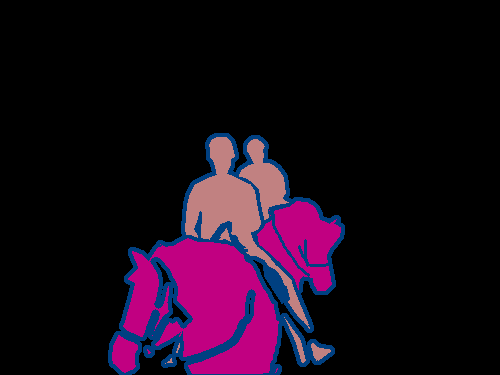}\\
        \includegraphics[width=1\columnwidth]{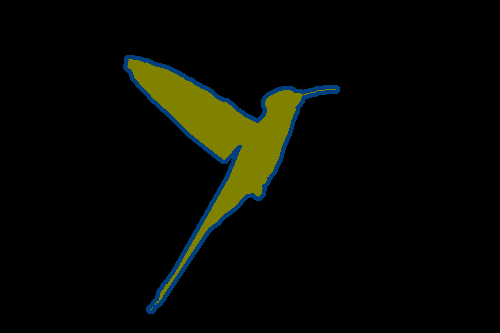} \\
         \includegraphics[width=1\columnwidth]{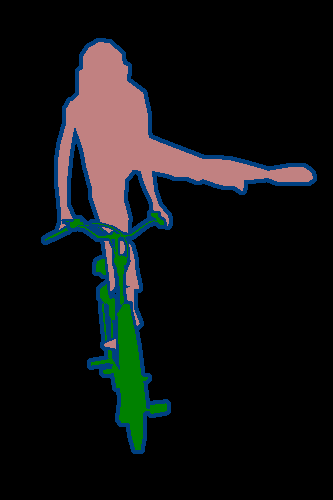} \\
        \includegraphics[width=1\columnwidth]{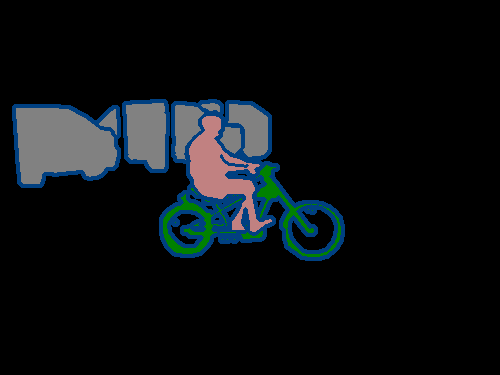} \\
        \includegraphics[width=1\columnwidth]{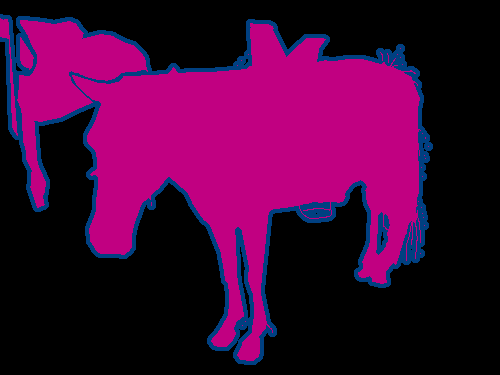} \\
        \includegraphics[width=1\columnwidth]{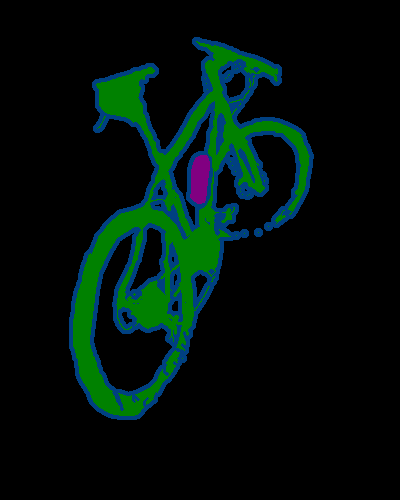} \\
        \includegraphics[width=1\columnwidth]{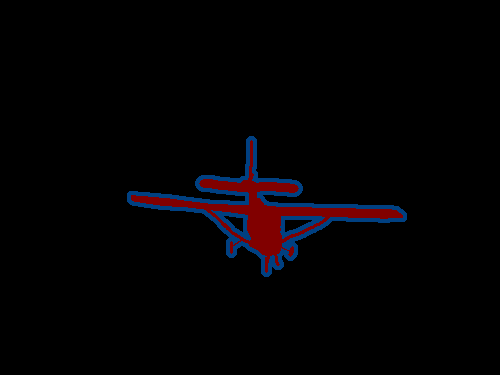} \\
        \includegraphics[width=1\columnwidth]{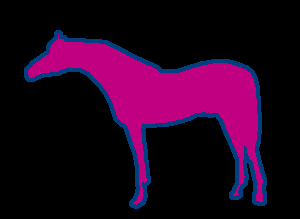}\\
        \includegraphics[width=1\columnwidth]{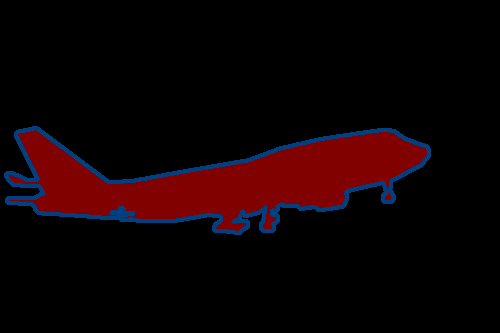}
        \captionsetup{labelformat=empty}
        \caption{ground truth}
    \end{subfigure}

   \caption{Examples on PASCAL VOC \textit{val} set. Kernel cut as regularization loss gives qualitatively better result than that with normalized cut loss. We found kernel cut results to have better edge alignment.}
   \label{fig:scribblesexamples}
\end{figure}

\subsection{Direct loss vs proposal generation}
\label{sec:directvsproposal}

Here we compare our direct loss and proposal generation methods (Sec. \ref{sec:proposal-connection}) in weakly
supervised setting mainly focusing on scribbles. Proposals can be generated \textit{offline} or \textit{online}. 
One straightforward proposal method is to treat GrabCut output as ``fake'' ground truth for training. 
ScribbleSup \cite{scribblesup} refines GrabCut output using network predicted segmentation as unary potentials. 
The proposals are updated but are generated offline. By online proposal generation, we let network output 
go through a CRF inference layer during training at each iteration. 
The loss for proposal generation is the cross entropy between the input and output of the CRF inference layer, 
see Sec.\ref{sec:proposal-connection}. A recent work that generates proposals online for tag-based
weakly-supervised segmentation is SEC \cite{kolesnikov2016seed}.

\begin{table}[h!]
\centering
\begin{tabular}{l | c| c | c| c |c}
\hline
    &  \multicolumn{4}{c|}{Weak} & \multirow{4}{*}{ Full}\\\cline{2-5}
    &  \multicolumn{3}{c|}{proposal generation} & direct loss\\\cline{2-5}
           &GrabCut & ScribbleSup &  SEC$^*$ &  \multirow{2}{*}{ CRF loss} &\\
                      &(one time) & (iterative) &  (online) &  &\\\hline
{DeepLab-MSc-largeFOV}        & 55.5 & n/a & {61.3} &\textbf{63.1} & 64.1\\
{DeepLab-MSc-largeFOV+CRF}        & 59.7 & 63.1 & {65.4} &\textbf{65.9} & 68.7\\
{DeepLab-VGG16}        & 59.0 & n/a  & {63.4} & \textbf{64.4}& 68.8\\
{DeepLab-ResNet101  }  & 63.9 & n/a  &{72.5} &\textbf{72.9}& 75.6\\\hline
\end{tabular}
\caption{Results using weak supervision (scribbles). The baseline is training with interactive GrabCut output. 
ScribbleSup  \cite{scribblesup}  alternates between GrabCut and CNN training, but the proposals are generated offline. It helps to have frequent online proposal updates at each iteration of training as in SEC$^*$, our adaptation of tag-based SEC \cite{kolesnikov2016seed} to weak supervision with
scribbles in \cite{scribblesup}. The best (quality and speed) training is based on simple direct loss
optimization avoiding proposal generations. This comparison uses the same dense CRF Gaussian bandwidths.}
\label{tab:scribblelampert}
\end{table}

 Table \ref{tab:scribblelampert} compares our direct loss method to proposal generation variants above.
We used the public implementation of SEC's {\em constrain-to-boundary loss}\footnote{https://github.com/kolesman/SEC} 
that combines explicit dense CRF proposal layer and cross entropy loss between the proposal and network output. 
We report the results for SEC$^*$, our adaptation of tag-based SEC to weak-supervision with scribbles 
from \cite{scribblesup}. 
We find that (frequent) online proposal updates give better results than 
those with fixed proposals. Compared to our direct loss method, (online) proposal generation gives inferior 
segmentation accuracy over different networks, see Table \ref{tab:scribblelampert}.

We further evaluate online proposal generation. Figure \ref{fig:directlossvstrustregion} compares it to our regularized loss method in terms of segmentation accuracy and obtained loss values. Even though the proposal generation scheme indirectly minimizes our regularized loss, such training scheme gives higher loss values than those 
obtained with our direct loss minimization. Also, direct loss minimization gives higher mIOUs for the training and validation.

\begin{figure}[b]
    \centering
        \includegraphics[width=0.24\textwidth]{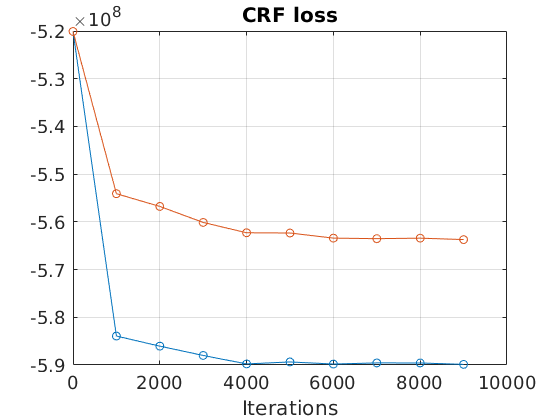}
        \includegraphics[width=0.24\textwidth]{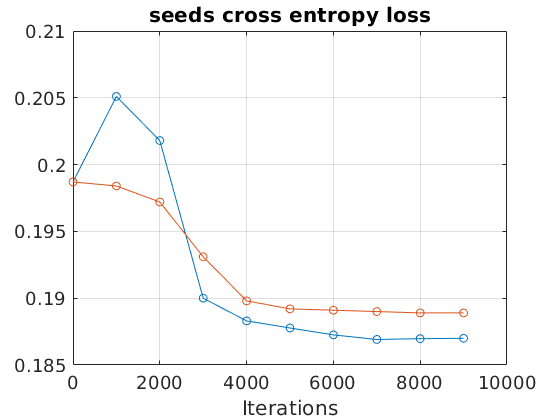}
         \includegraphics[width=0.24\textwidth]{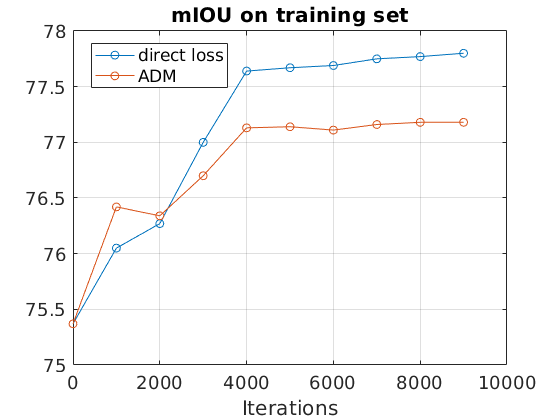}
        \includegraphics[width=0.24\textwidth]{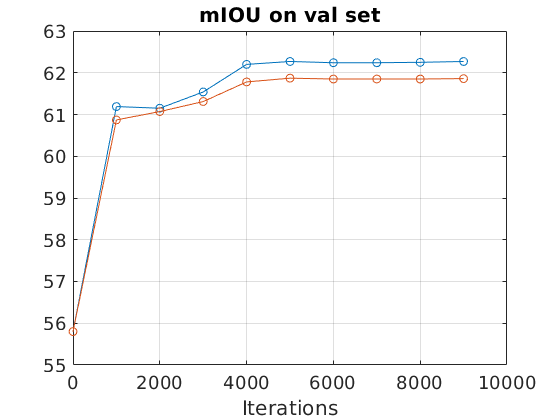}
   \caption{Our direct loss scheme achieves better mIOU accuracy on training and val set. The CRF loss and seeding loss of our trained model are also less than that with proposal generation scheme. For fair comparison, our CRF loss and the CRF inference layer in proposal generation method have the same Gaussian kernel in this experiment.}
   \label{fig:directlossvstrustregion}
\end{figure}

\begin{figure}[th!]
    \centering
    \begin{subfigure}{0.15\textwidth}
    \centering
        \includegraphics[width=1\columnwidth]{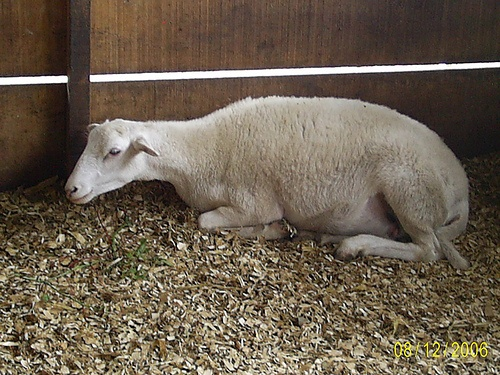}\\
        \includegraphics[width=1\columnwidth]{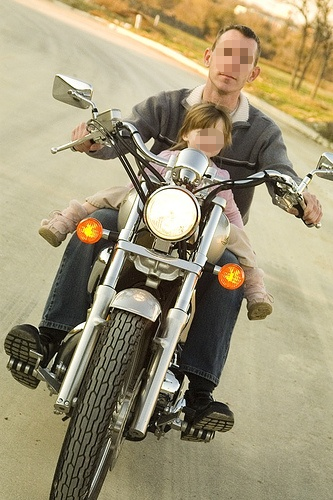} \\
        \includegraphics[width=1\columnwidth]{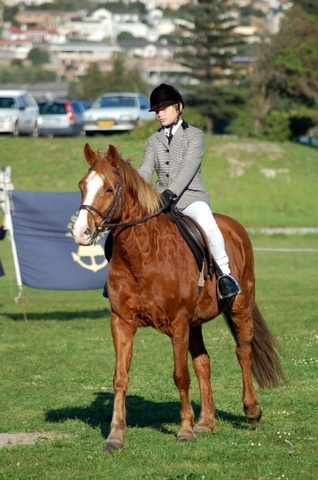} \\
        \includegraphics[width=1\columnwidth]{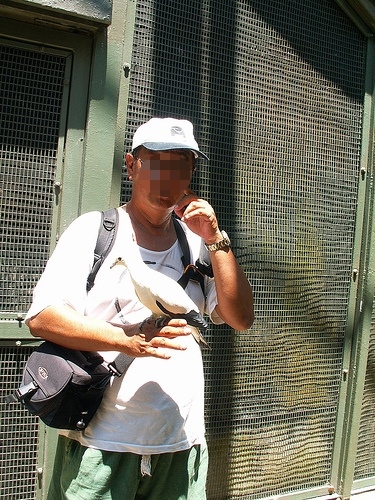} \\
        \includegraphics[width=1\columnwidth]{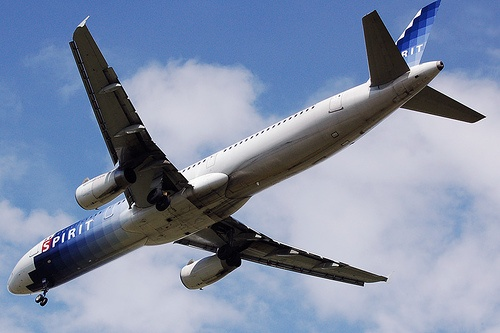}
        \captionsetup{labelformat=empty}
        \caption{image}
    \end{subfigure}
            \begin{subfigure}{0.15\textwidth}
    \centering
        \includegraphics[width=1\columnwidth]{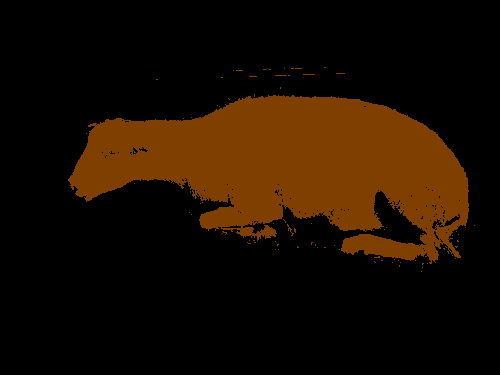}\\
        \includegraphics[width=1\columnwidth]{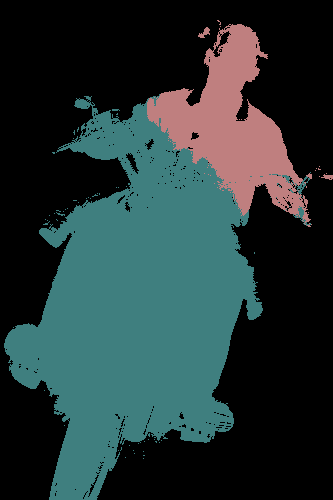} \\
        \includegraphics[width=1\columnwidth]{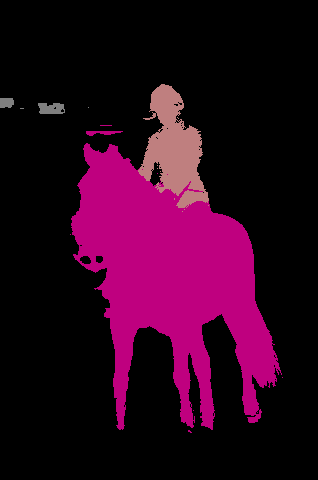} \\
        \includegraphics[width=1\columnwidth]{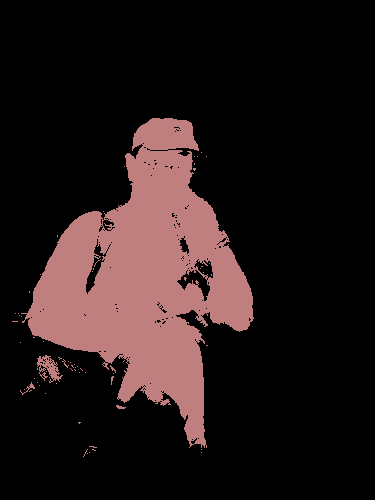} \\
        \includegraphics[width=1\columnwidth]{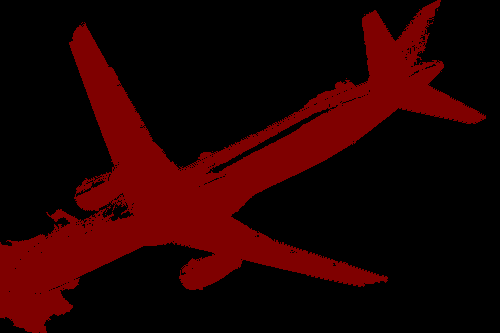}
        \captionsetup{labelformat=empty}
        \caption{SEC \cite{kolesnikov2016seed}}
    \end{subfigure}
        \begin{subfigure}{0.15\textwidth}
    \centering
        \includegraphics[width=1\columnwidth]{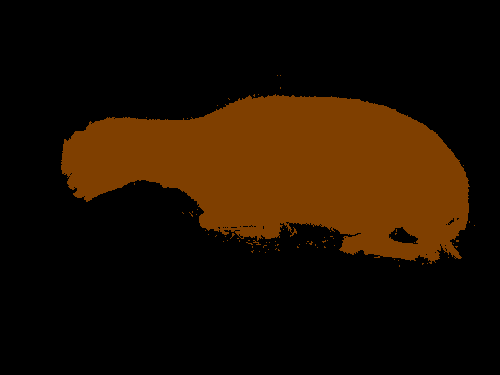}\\
        \includegraphics[width=1\columnwidth]{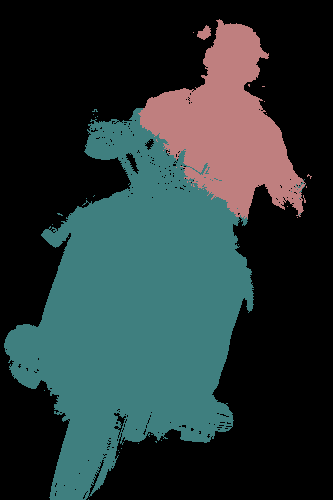} \\
        \includegraphics[width=1\columnwidth]{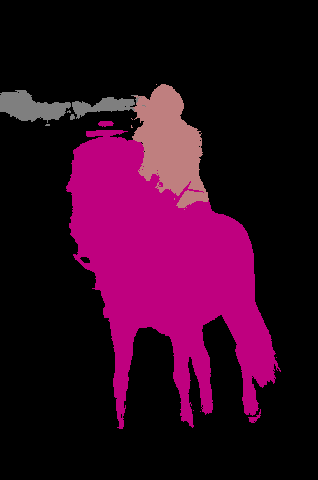} \\
        \includegraphics[width=1\columnwidth]{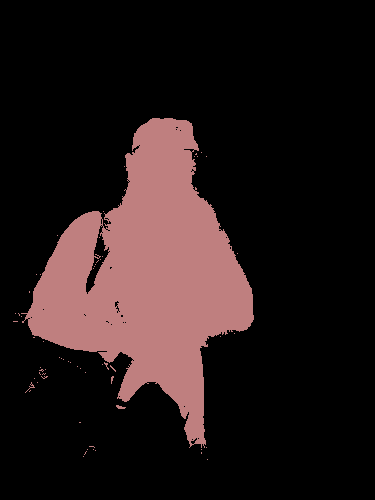} \\
        \includegraphics[width=1\columnwidth]{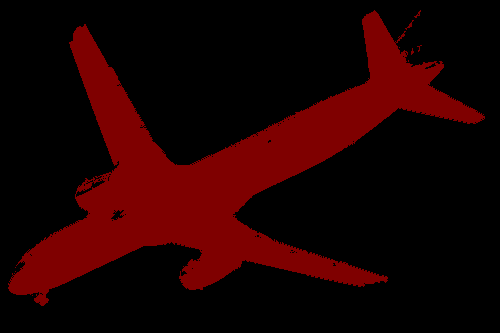}
        \captionsetup{labelformat=empty}
        \caption{w/ CRF loss}
    \end{subfigure}
        \begin{subfigure}{0.15\textwidth}
    \centering
        \includegraphics[width=1\columnwidth]{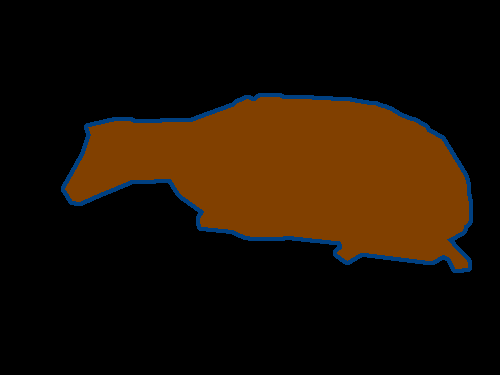}\\
        \includegraphics[width=1\columnwidth]{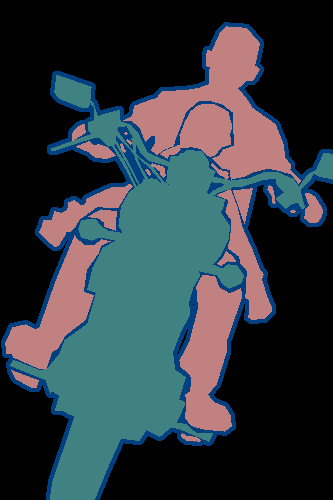} \\
        \includegraphics[width=1\columnwidth]{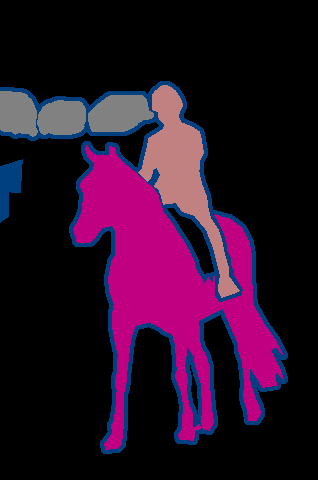} \\
        \includegraphics[width=1\columnwidth]{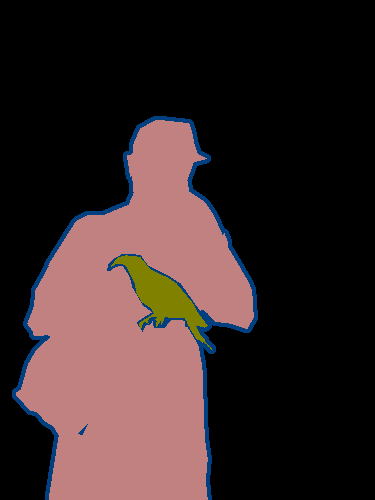} \\
        \includegraphics[width=1\columnwidth]{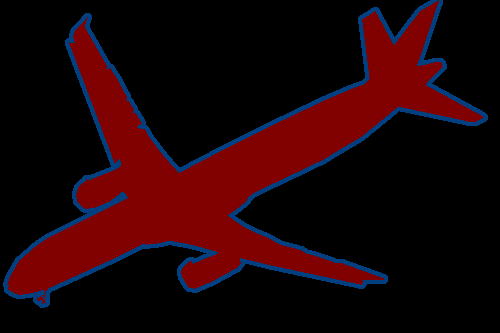}
        \captionsetup{labelformat=empty}
        \caption{ground truth}
    \end{subfigure}

   \caption{Examples on PASCAL VOC \textit{val} set for supervision with image-level labels (tags). We train using the seeding loss, expansion loss in SEC \cite{kolesnikov2016seed} and our CRF loss. Similar segmentation is obtained yet we avoid any iterative CRF inference and have the direct loss instead.}
   \label{fig:tagexamples}
\end{figure}

As mentioned earlier, SEC \cite{kolesnikov2016seed} was originally focused on tag-based supervision
and Table \ref{tab:sectag} reports some tests for that form of weak supervision. 
We compare SEC with its simplification replacing
their constrain-to-boundary loss by our direct regularization loss. 
We train using different combinations of losses for supervision based on image-level labels/tags. 
Our CRF loss helps to improve training to 43.9\% compared to 38.4\% without it. 
There is only small improvement in segmentation mIOU when replacing constrain-to-boundary loss by CRF loss.
However, the direct loss layer is several times faster than SEC integrating explicit proposal layer. 
The segmentation accuracy and overall training speed are also reported in Tab. \ref{tab:sectag}. 
The results are for the DeepLab-largeFOV network since it is fast to train. 
We also tested a variant of SEC without (CRF) proposal layer back-propagation, which
we show is redundant in practice.

Fig. \ref{fig:tagexamples} shows testing examples for our method and SEC with image tags as supervision. Using direct loss rather than the constrain-to-boundary loss gives similar segmentation, while being faster to train since no inference is needed.

\begin{table}[t]
\centering
\begin{tabular}{| l | l | c |c| c| c |}
\hline
\multicolumn{2}{|c|}{}   & \multicolumn{4}{|c|}{include this loss?}                  \\ \hline
\multirow{ 4}{*}{Losses} & Seeding loss \cite{kolesnikov2016seed}    & \checkmark    & \checkmark    &                       \checkmark&            \checkmark                \\
&Expansion loss \cite{kolesnikov2016seed}    & \checkmark    & \checkmark    &   \checkmark                    &                      \checkmark       \\
&Constrain-to-boundary loss \cite{kolesnikov2016seed}    &     & \checkmark    &     $\star$                 &                                \\
&Our direct CRF loss    &     &     &                      &        \checkmark                     \\ \hline
\multicolumn{2}{|c|}{mIOU (\%)}   & 38.4    & 43.7    &        43.8              &        43.9                 \\ \hline
\multicolumn{2}{|c|}{Overall training time  in s/batch}     & 0.86   & 1.19 (0.33)   &  1.19  (0.33)                   &  0.98 (0.12)          \\ \hline
\end{tabular}
\caption{Tag-based weak supervision. We train with different combinations of the losses in SEC \cite{kolesnikov2016seed} and our CRF loss. Replacing the constrain-to-boundary loss in SEC \cite{kolesnikov2016seed} by direct CRF loss gives minor improvement in accuracy, but training with our direct loss is faster since no iterative CRF inference is needed. 
We also compare to a variant ($\star$) of SEC without back-propagation of the CRF inference layer. 
Parenthesis $(\cdot)$ show the computational times for the constrain-to-boundary loss layer or our direct loss layer. 
}
\label{tab:sectag}
\end{table}

To see the limit of our algorithm with scribble supervision, we train with shortened scribbles visualized in Fig. \ref{fig:scribblevis}. Note that with length zero, there is only one click or spot for each object. For different length ratios from zero to 100\%, our direct loss method achieved much better segmentation than ScribbleSup \cite{scribblesup}, see Fig. \ref{fig:shortscribblecurve}. The improvement over ScribbleSup \cite{scribblesup} is more significant for shorter scribbles or even clicks.

\begin{figure}[h!]
    \centering
    \captionsetup[subfigure]{labelformat=empty}
      \begin{subfigure}{0.22\textwidth}
        \includegraphics[width=1\textwidth]{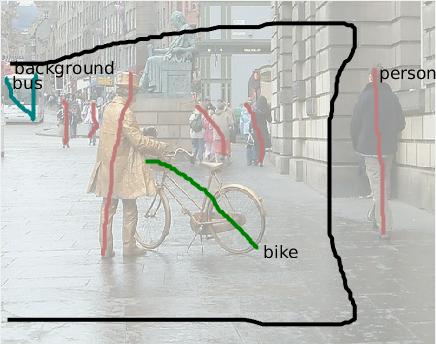}
        \caption{length 100\%}
      \end{subfigure}
       \begin{subfigure}{0.22\textwidth}
        \includegraphics[width=1\textwidth]{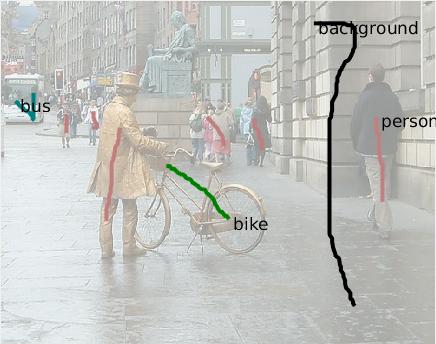}
        \caption{length 50\%}
      \end{subfigure}
       \begin{subfigure}{0.22\textwidth}
        \includegraphics[width=1\textwidth]{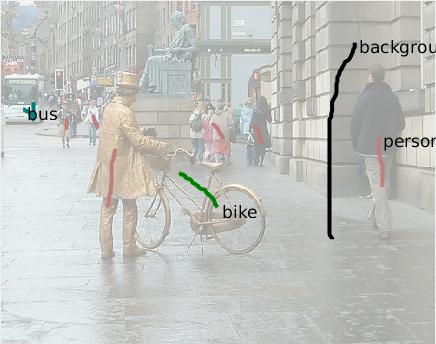}
        \caption{length 30\%}
      \end{subfigure}
       \begin{subfigure}{0.22\textwidth}
        \includegraphics[width=1\textwidth]{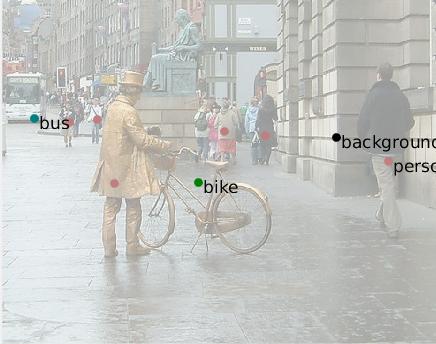}
        \caption{length 0\% (click)}
      \end{subfigure}
   \caption{Similar to \cite{scribblesup}, we shorten the scribbles with different length ratios. With length zero (clicks) is the most challenging case for training.}
   \label{fig:scribblevis}
\end{figure}

\begin{figure}[b!]
    \centering
    \includegraphics[width=0.34\textwidth]{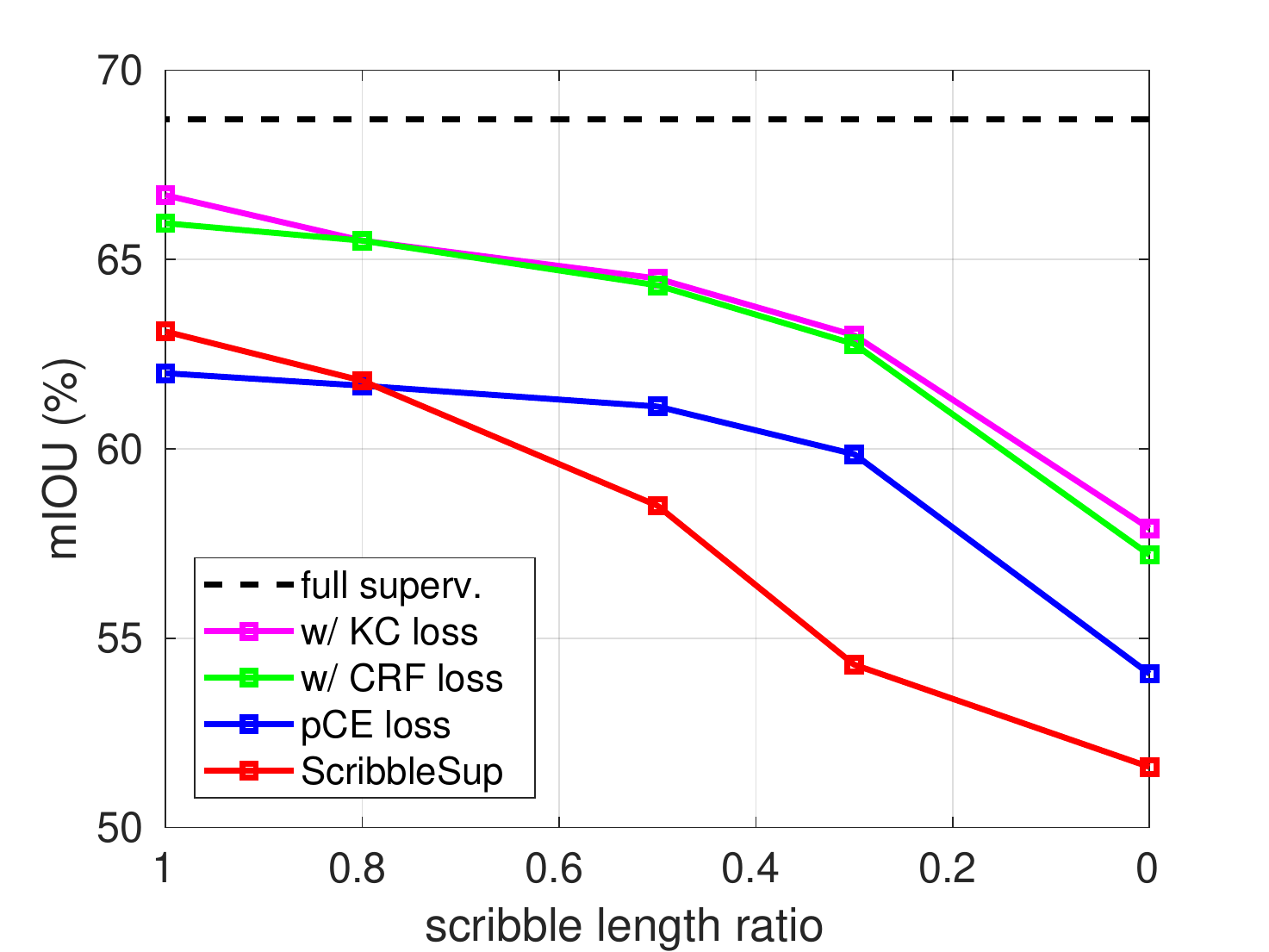}
   \caption{mIOU on \textit{val} set when train with shorter scribbles or clicks (length zero).}
   \label{fig:shortscribblecurve}
\end{figure}

\subsection{Fully and semi supervised segmentation}
\label{sec:fullyandsemi}
We've demonstrated the usefulness of regularized loss for weakly supervised segmentation. Here we test if it also helps full supervision or semi-supervision with extra unlabeled images. For full supervision, we add NC loss on labeled masks besides the cross entropy loss. This experiment is on a simple saliency dataset \cite{ChengPAMI} where color clustering is obvious and likely to help. As shown in Tab. \ref{tab:fullsupervision}, when we increase the weight of $R_{NC}(S)$, we indeed obtained segmentation that is more regularized. However, with extra regularization loss during training, the cross entropy loss got worse and mIOU decreased. The conclusion is that imposing regularized loss naively on labeled images doesn't help fully supervision segmentation. Empirical risk minimization is in some sense optimal for fully labeled data. Extra regularization loss steers the network in the wrong direction if the regularization doesn't totally agree with the ground truth. Reporting this result though negative helps to complete our investigation of regularized loss for fully, weakly and semi-supervised settings.

\begin{table}[t]
\centering
\begin{tabular}{| c | c | c | c |}
\hline
 NC loss weight    &        mIOU              &        cross entropy loss        & NC loss          \\ \hline
  0    &        89.85\%             &        0.106        & 0.536          \\ \hline
   0.05    &        89.53\%              &      0.108        & 0.526          \\ \hline
    0.1   &        89.38\%             &        0.110        & 0.517          \\ \hline
     0.2    &        89.39\%             &        0.112       & 0.509          \\ \hline
      0.5    &        88.75\%             &        0.125       & 0.485          \\ \hline
\end{tabular}
\caption{Negative effect of regularization loss for full supervision.}
\label{tab:fullsupervision}
\end{table}

For training with both labeled images and unlabeled images, our joint losses include cross entropy on labeled images and regularization on unlabeled ones. The 11K labeled images are from PASCAL VOC 2012 and the 10K unlabeled ones are from VOC 2007. We train DeepLab-LargeFOV with different amount of labeled \& unlabeled images, see Tab. \ref{tab:semisupervision}. For the baseline that can only utilize labeled images, the performance degrades with less masks, as expected. For our framework, the labeled and unlabeled images are mixed and randomly sampled in each batch. We observed 0.7\%~1.9\% improvement with our regularized loss. Note that this result is highly preliminary and detailed analysis of overfitting, generalization property and comparison to recent semi-supervised segmentation \cite{hung2018adversarial} with extra unlabeled images will be our future work.

\begin{table}[b]
\centering
\begin{tabular}{| c | l | c | c |c| c| c |}
\hline
\multirow{ 2}{*}{training data} & \# of labeled images&  11K & 11K    & 7K    &        5K              &        3K                 \\ \cline{2-7}
&\# of unlabeled images &  10K  & 0    & 4K    &  6K                     &  8K          \\ \hline
\multirow{ 2}{*}{losses}  &cross entropy only &  63.5\%  & 63.5\%    & 61.5\%    &  60.1\%                     &  57.6\%          \\ \cline{2-7}
&cross entropy + CRF reg. &  64.6\%  & 63.5\%    & 63.4\%    &  61.8\%                     &  58.3\%          \\ \hline
\end{tabular}
\caption{Our regularization loss $R_{CRF}(S)$ on unlabeled images help to improve semi-supervised segmentation.}
\label{tab:semisupervision}
\end{table}
\section{Conclusion and Future Work}
\label{conclusion}
\textit{Regularized semi-supervised loss} is a principled approach to semi-supervised deep learning \cite{weston2012deep,goodfellow2016deep}, in general. We utilize such principle for weakly supervised CNN segmentation. In particular, this paper is continuation of the study of losses motivated by standard shallow segmentation \cite{ncloss:cvpr18}. While \cite{ncloss:cvpr18} is entirely on normalized cut loss, in this paper we propose and evaluate several {regularized loss} for weakly-supervised CNN segmentation based on Potts/CRF \cite{BVZ:PAMI01,koltun:NIPS11}, normalized cut \cite{Shi2000} and KernelCut \cite{NC-MRF:eccv16} regularizer. DenseCRF \cite{koltun:NIPS11} is very popular as post-processing \cite{deeplab} or trainable layer \cite{CRFlayersjournal} for CNN segmentation. We are the first to use a relaxed version of DenseCRF directly as part of the loss.

In contrast to our direct regularized loss approach, the main stream in weakly supervised segmentation rely on generating "fake" full masks from partial input and train a network to match the proposals \cite{scribblesup,Chandraker2017,simpledoesit,kolesnikov2016seed,papandreou2015weakly,dai2015boxsup}. Proposals can be pre-computed or iteratively updated. Some work even back-propagate the proposal generation step \cite{Chandraker2017,kolesnikov2016seed}. We show that proposal methods can be viewed as approximate alternating direction method (ADM) for optimization of our direct loss. Using direct loss gives better optimization while being more efficient than proposal generation scheme since no CRF inference is needed.

This paper pushes the limit of weakly-supervised segmentation. Comprehensive experiments (Sec.\ref{sec:experiments}) with our regularized weakly supervised losses  show
(1) state-of-the-art performance for weakly supervised CNN segmentation reaching near full-supervision accuracy and
(2) better quality and efficiency than proposal generating methods or normalized cut loss \cite{ncloss:cvpr18}.
Alternating schemes (proposal generation) give higher loss at convergence. Besides for weak supervision, we also report our preliminary results for full and semi-supervision with unlabeled images.

In principle, any differentiable loss function fits our regularized loss framework. Exploring other relaxations of CRF as losses \cite{ChambolleDarbon:TV2009,Nikolova:SIAM06,pock:CVPR09,Couprie:PAMI11,kumar2016,kumar2017} and corresponding efficient gradient computation is left for future work. Also it would be interesting to apply our CRF regularized loss framework for weakly-supervised computer vision problems other than segmentation.

\appendix
\section{Mean-field inference for DenseCRF}
\label{sec:meanfield}
Here we show that the iterative parallel mean-field inference \cite{koltun:NIPS11} indeed minimizes \eqref{proposal-CRF-negative-entropie} with pairwise DenseCRF regularizer and unary potentials $\tilde{S}_p$ (e.g. given by network). 
\begin{equation}
E(X) =   \sum_{p} {H}(X_p,\tilde{S}_p) + \lambda R_{CRF}(X) -\sum_{p} H(X_p). \nonumber
\end{equation}
For positive semidefinite affinity matrix $W$, e.g. with Gaussian Kernel, 
$$R_{CRF}(X)\;\;=\;\;\sum_{k} {X^{k'}{W}({\mathbf1}-X^k)}\;\;\utc\;\; -\sum_{k} {X^{k'}{W}X^k}$$ 
is concave\footnote{\utc means up to an additive constant.}. 
Since the cross entropy $H(X_p,\tilde{S}_p)$ is linear and the negative entropy $-H(X_p)$ is convex w.r.t. $X_p$, 
the concave-convex procedure (CCCP) can iteratively solve an approximation of $E(X)$ by linearizing the concave part at $\tilde{S}$.
\begin{eqnarray*}
A(X) &=&   \sum_{p} {H}(X_p,\tilde{S}_p) + \lambda\langle X, \triangledown R_{CRF}(X)|_{\tilde{S}} \rangle   -\sum_{p} H(X_p) . \nonumber \\
 &=&  \sum_{p} {H}(X_p,\tilde{S}_p) -2\lambda \sum_p \sum_k X_p^k \cdot [W\tilde{S}^k]_p   -\sum_{p} H(X_p). \nonumber \\
  &=& -\sum_k X_p^k \cdot \log \tilde{S}_p^k -2\lambda \sum_k X_p^k \cdot [W\tilde{S}^k]_p  -\sum_{p} H(X_p) . \nonumber \\
\end{eqnarray*}
KKT approach for minimizing $A(X)$ subject to probability simplex constraints $\sum_k X_p^k=1$ yields the following optima,
\begin{equation}
\label{eq:meanfieldupdate}
\arg \min_X A(X) = \frac{1}{z_p}\exp\{-\log \tilde{S}_p^k - 2\lambda[W\tilde{S}^k]_p \},
\end{equation}
where $z_p$ is a normalization constant for softmax. \eqref{eq:meanfieldupdate} is exactly the mean-field update for dense CRF \cite{koltun:NIPS11}. 
Note that the updates \eqref{eq:meanfieldupdate} is also justified in a similar way in \cite{Kraehenbuehl2013} for convergent optimization of 
KL distance between factorial marginal distribution and Gibbs distribution induced by CRF. 
Our justification of \eqref{eq:meanfieldupdate} is different. We show alternative interpretation of 
mean-field updates \eqref{eq:meanfieldupdate} as minimizing CRF potential $R_{CRF}(X)$ plus negative entropy $-H(X)$.

\bibliographystyle{splncs}
\bibliography{egbib}
\end{document}